\documentclass[11pt]{article}

\usepackage{fullpage}
\usepackage{amsthm,amsmath,amssymb}
\usepackage{xcolor}
\usepackage{algpseudocode}
\usepackage{algorithm}
\usepackage{graphicx}
\usepackage{subcaption}
\usepackage{tikz}
\newsavebox{\theorembox}
\newsavebox{\factbox}
\newsavebox{\lemmabox}
\newsavebox{\corollarybox}
\newsavebox{\propositionbox}
\newsavebox{\examplebox}
\newsavebox{\conjecturebox}
\newsavebox{\algbox}
\newsavebox{\qbox}
\newsavebox{\problembox}
\newsavebox{\remarkbox}
\newsavebox{\definitionbox}
\newsavebox{\assumptionbox}
\newsavebox{\hypothesisbox}
\savebox{\theorembox}{\noindent\bf Theorem}
\savebox{\factbox}{\noindent\bf Fact}
\savebox{\lemmabox}{\noindent\bf Lemma}
\savebox{\corollarybox}{\noindent\bf Corollary}
\savebox{\propositionbox}{\noindent\bf Proposition}
\savebox{\examplebox}{\noindent\bf Example}
\savebox{\conjecturebox}{\noindent\bf Conjecture}
\savebox{\algbox}{\noindent\bf Algorithm}
\savebox{\qbox}{\noindent\bf Question}
\savebox{\definitionbox}{\noindent\bf Definition}
\savebox{\problembox}{\noindent\bf Problem}
\savebox{\assumptionbox}{\noindent\bf Assumption}
\savebox{\hypothesisbox}{\noindent\bf Hypothesis}
\savebox{\remarkbox}{\noindent\bf Remark}
\newtheorem{theorem}{\usebox{\theorembox}}
\newtheorem{lemma}[theorem]{\usebox{\lemmabox}}

\newtheorem{proposition}[theorem]{\usebox{\propositionbox}}

\newtheorem{question}{\usebox{\qbox}}

\newtheorem{hypothesis}{\usebox{\hypothesisbox}}
\newtheorem{definition}{\usebox{\definitionbox}}
\newtheorem{remark}{\usebox{\remarkbox}}
\newtheorem{fact}{\usebox{\factbox}}
\newcommand{\tS}{\tilde{S}}

\newcommand{\bx}{\mathbf{x}}

\newcommand{\bw}{\mathbf{w}}

\newcommand{\bz}{\mathbf{z}}

\newcommand{\ba}{\mathbf{a}}
\newcommand{\bA}{\mathbf{A}}
\newcommand{\bs}{\mathbf{s}}
\newcommand{\cS}{\mathcal{S}}
\newcommand{\hS}{\hat{S}}
\newcommand{\hbw}{\hat{\bw}}
\newcommand{\bu}{\mathbf{u}}

\newcommand{\ty}{\tilde{y}}
\newcommand{\tm}{\tilde{m}}
\newcommand{\tbx}{\tilde{\bx}}
\newcommand{\tx}{\tilde{x}}
\newcommand{\tn}{\tilde{n}}
\newcommand{\cD}{\mathcal{D}}
\newcommand{\poly}{\text{poly}}
\newcommand{\cC}{\mathcal{C}}
\newcommand{\cX}{\mathcal{X}}
\newcommand{\cY}{\mathcal{Y}}
\newcommand{\cF}{\mathcal{F}}
\newcommand{\cR}{\mathcal{R}}
\newcommand{\cL}{\mathcal{L}}
\newcommand{\cW}{\mathcal{W}}
\newcommand{\cA}{\mathcal{A}}

\newcommand{\hbx}{\hat{\bx}}
\newcommand{\hy}{\hat{y}}
\newcommand{\hm}{\hat{m}}
\newcommand{\hn}{\hat{n}}
\newcommand{\R}{\mathbb{R}}
\newcommand{\tbw}{\tilde{\bw}}
\newcommand{\relu}{\textsc{relu}}

\newcommand{\E}{\mathbb{E}}
\newcommand{\N}{\mathbb{N}}
\newcommand{\be}{\mathbf{e}}
\newcommand{\mmcs}{\textsc{MMCS}}

\newcommand{\cB}{\mathcal{B}}
\newcommand{\cT}{\mathcal{T}}
\newcommand{\card}{\text{card}}
\newcommand{\bzero}{\mathbf{0}}
\newcommand{\bone}{\mathbf{1}}
\newcommand{\edge}{\text{edge}}
\newcommand{\kap}{\kappa}
\newcommand{\eps}{\epsilon}
\renewcommand{\varepsilon}{\eps}

\DeclareMathOperator{\opt}{OPT}
\DeclareMathOperator{\den}{den}
\DeclareMathOperator{\sgn}{sgn}
\newcommand{\optmmcs}{\opt_{\mmcs}}

\title{Tight Hardness Results for Training Depth-2 ReLU Networks\footnote{This work subsumes our earlier manuscript~\cite{MR18}.}}

\author{Surbhi Goel\thanks{Microsoft Research NYC. Email: \texttt{goel.surbhi@microsoft.com}. Work was done while the author was a PhD student at UT Austin and was supported by the JP Morgan AI Research PhD Fellowship.}
\and Adam Klivans\thanks{UT Austin. Email: \texttt{klivans@cs.utexas.edu}. Supported by NSF awards AF-1909204, AF-1717896, and the NSF AI Institute for Foundations of Machine Learning (IFML). Work done while visiting the Institute for Advanced Study, Princeton, NJ.}
\and Pasin Manurangsi\thanks{Google Research. Email: \texttt{pasin@google.com}. Part of this work was done while the author was at UC Berkeley and was partially supported by NSF under Grants No. CCF 1655215 and CCF 1815434.}
\and Daniel Reichman\thanks{WPI. Email: \texttt{daniel.reichman@gmail.com}.}}

\begin{document}

\maketitle
\begin{abstract}

We prove several hardness results for training depth-2 neural networks with the ReLU activation function; these networks are simply weighted sums (that may include negative coefficients) of ReLUs. 
Our goal is to output a depth-2 neural network that minimizes the square loss with respect to a given training set.
We prove that this problem is NP-hard already for a network with a single ReLU.  We also prove NP-hardness for outputting a weighted sum of $k$ ReLUs minimizing the squared error (for $k>1$) even in the realizable setting (i.e., when the labels are consistent with an unknown depth-2 ReLU network).
We are also able to obtain lower bounds on the running time
in terms of the desired additive error $\epsilon$. To obtain our lower bounds, we use the Gap Exponential Time Hypothesis (Gap-ETH) as well as a new hypothesis regarding the hardness of approximating the well known Densest $\kap$-Subgraph problem in subexponential time (these hypotheses are used separately in proving different lower bounds).
For example, we prove that under reasonable hardness assumptions, any {\em proper} learning algorithm for finding the best fitting ReLU must run in time exponential in $1/\eps^2$. Together with a previous work regarding improperly learning a ReLU~\cite{goel2016reliably}, this implies the first separation between proper and improper algorithms for learning a ReLU. We also study the problem of properly learning a depth-2 network of ReLUs with bounded weights giving new (worst-case) upper bounds on the running time needed to learn such networks both in the realizable and agnostic settings. Our upper bounds on the running time essentially matches our lower bounds in terms of the dependency on $\eps$.

\end{abstract}
\section{Introduction}
Neural networks have become popular in machine learning tasks arising in multiple applications such as computer
vision, natural language processing, game playing and robotics \cite{lecun2015deep}. One attractive feature of neural networks is being universal approximations: a network with a
single hidden layer\footnote{We also refer to such networks as depth-2 networks or shallow networks.} with sufficiently many neurons can approximate arbitrary well any measurable
real-valued function~\cite{hornik1989multilayer,cybenko1989approximation}. These networks are typically
trained on labeled data by setting the weights of the units to minimize the loss function (often the squared loss is used)
over the training data. The challenge is to find a computationally efficient way to set the weights to achieve low error. While heuristics such as stochastic
gradient descent (SGD) have been successful in practice, our theoretical understand about the amount of running-time needed to train neural networks is still lacking.

It has been known for decades~\cite{blum1989training,megiddo1988complexity,judd1988complexity} that finding a set of weights that minimizes the loss of the training set is NP-hard.  These hardness results, however, only apply to classification problems and to settings where the neural networks involved use discrete, Boolean activations. 
Our focus here is on neural networks with real inputs whose neurons have the real-valued ReLU activation function. 
Specifically, we consider depth-2 networks of ReLUs, namely either a single ReLU or a weighted sum of ReLUs \footnote{We also assume all the biases of the units are 0.}, and the optimization problem of training them giving labelled data points, which are defined below.

\begin{definition}\label{def:recteifier}
A \emph{rectifier} is the real function $[x]_{+} := \max(0,x)$.
A \emph{rectified linear unit (ReLU)} is a function  $f(\bx):\mathbb{R}^n\rightarrow \mathbb{R}$ of the form $f(\bx)= [\langle\bw,\bx\rangle]_+$ where $\bw \in \mathbb{R}^n$ is fixed.
A \emph{depth-2 neural network with $k$ ReLUs} (abbreviated as \emph{$k$-ReLU}) is a function from $\mathbb{R}^n$ to $\mathbb{R}$ defined by
$$\relu_{\bw^1, \dots, \bw^k, \ba}(\bz) =\sum_{j=1}^k a_j[\langle\bw^j,\bz\rangle]_+.$$
Here $\bx \in \mathbb{R}^n$ is the input, $\ba = (a_1, \dots, a_k) \in \{-1, 1\}^k$ is a vector of ``coefficients'', $\bw^j=(w^j_1,\ldots, w^j_n)\in \mathbb{R}^n$ is a weight vector associated with the $j$-th unit.
When $a_1 = \dots = a_k = 1$, we refer to $\relu_{\bw^1, \dots, \bw^k, \ba}(\bz)$ as the sum of $k$ ReLUs, and we may omit $\ba$ from the subscript.
\end{definition}
We note that the assumption that $a_1, \dots, a_k \in \{+1, -1\}$ is without loss of generality (e.g., \cite{pan2016expressiveness}): for any non-zero $a_1, \dots, a_k \in \R \setminus \{0\}$ and $\bw^1, \dots, \bw^k$, we may consider $\hat{a}_1 = \frac{a_1}{|a_1|}, \dots, \hat{a}_k = \frac{a_k}{|a_k|}$ and $\hat{\bw}^1 = |a_1| \bw^1, \dots, \hat{\bw}^k = |a_k| \bw^k$ instead, which represent the same depth-2 network of $k$ ReLUs. 

When training neural networks composed of ReLUs, a popular method is to find, given training data, a set of coefficients and weights for each gate minimizing the squared loss. 
\begin{definition}
Given a set of $m$ samples $\bx_1, \ldots ,\bx_m \in \mathbb{R}^n$ along with $m$ labels $y_1, \ldots, y_m \in \mathbb{R}$, our goal is to find $\bw^1,\ldots \bw^k,\ba$
which minimize the average squared training error of the sample, i.e.,
\begin{equation}\label{equation:relumin}
 \min_{\bw^1,\ldots, \bw^k,\ba}\frac{1}{m} \sum_{i=1}^m (\relu_{\bw^1, \dots, \bw^k, \ba}(\bx_i)-y_i)^2
\end{equation}
We refer to the optimization problem (\ref{equation:relumin}) as the \emph{$k$-ReLU training problem} (aka \emph{$k$-ReLU regression}).

When  $\bw^j=(w^j_1,\ldots, w^j_n)$ are assumed to have Euclidean norm at most $1$ and $y_i$ are assumed to be in $[-k, k]$, we refer to the optimization problem above as the \emph{bounded $k$-ReLU training problem}.
\end{definition}
Sometimes we assume that the ``coefficient'' vector $\ba$ is fixed in advance (and known to the optimizer) and not part of the input to the training problem. We mention this explicitly when relevant.
Also observe that in the optimization problem above we are looking for a {\bf global minimum} rather than a local minimum.
A multiset of samples $\{(\bx_i, y_i)\}_{i \in [m]}$ is said to be \emph{realizable} if there exist $\bw^1, \cdots, \bw^k,\ba$ which result in zero training error.

Our goal is to pin down the computational complexity of the training problem for depth-2 networks of ReLUs, by answering the following question:

\begin{question}
What is the worst-case running time of training a $k$-ReLU?
\end{question}

We focus on depth-2 networks which are rather involved and give rise to nontrivial algorithmic challenges~\cite{vempala2018polynomial,Woodruf}.
Understanding shallow networks seems to be a prerequisite for understanding the complexity of training networks of depth greater than $2$.


\subsection{Our results}

We first consider arguably the simplest possible network: a single ReLU. We show that, already for such a network, the training problem is  NP-hard. In fact, our result even rules out a large factor \emph{multiplicative} approximation of the minimum squared error, as stated below.

\begin{theorem}[Hardness of Training a single ReLU] \label{thm:single-short}
The 1-ReLU training problem is NP-hard. 
Furthermore, given a sample of $m$ data points of dimension $n$ it is NP-hard to approximate the 
optimal squared error within a multiplicative factor of $(nm)^{1/\poly \log \log (nm)}$.
\end{theorem}

Given such a strong multiplicative inapproximability result, a natural question is whether one can get a good algorithm for \emph{additive} approximation guarantee. Notice that we cannot hope for additive approximation in general, because scaling the samples and their labels can make the additive approximation gap arbitrarily large. Hence, we must consider the bounded 1-ReLU Training problem. For this, we give a simple $2^{O(1/\eps^2)} \poly(n, m)$ time algorithm with additive approximation $\eps$. Furthermore, it easily generalizes to the case of the bounded $k$-ReLU Training problem for $k > 1$, but we have to pay a factor of $k^5$ in the exponent:

\begin{theorem}[Training Algorithm] \label{thm:agnostic-training-short}
There is a (randomized) algorithm that can solve the bounded $k$-ReLU training problem to within any additive error $\epsilon > 0$ in time $2^{O(k^5/\eps^2)}\poly(n, m)$.
\end{theorem}

Perhaps more surprisingly, we can prove a tight running time lower bound for the bounded 1-ReLU training problem, which shows that the term $1/\eps^2$ in the exponent is necessary. Our running time lower bound relies on the assumption that there is no subexponential time algorithm for approximating the \emph{Densest $\kap$-Subgraph} problem within any constant (multiplicative) factor. Recall that, in the Densest $\kap$-Subgraph (D$\kap$S) problem, we are given a graph $G = (V, E)$ and a positive integer $\kap$. The goal is to select a subset $T \subseteq V$ of $\kap$ vertices that induces as many edges as possible. We use $\den_{\kap}(G)$ to denote this optimum\footnote{Equivalently, $\den_{\kap}(G) := \max_{T \subseteq V, |T| = \kap} |E(T)|$.} and $N$ to denote the number of vertices, $|V|$. Our hypothesis can be stated formally as follows.

\begin{hypothesis} \label{hyp:dks}
For every constant $C \geq 1$, there exist\footnote{As $C$ increases, $\delta$ and $d$ decreases.} $\delta = \delta(C) > 0$ and $d = d(C) \in \mathbb{N}$ such that the following holds. No $O(2^{\delta N})$-time algorithm can, given an instance $(G, \kap)$ of D$\kap$S where each vertex of $G$ has degree at most $d$ and an integer $\ell$, distinguish between the following two cases:
\begin{itemize}
\item (Completeness) $\den_{\kap}(G) \geq \ell$.
\item (Soundness) $\den_{\kap}(G) < \ell/C$.
\end{itemize}
\end{hypothesis}

While this hypothesis is new (we are the first to introduce it), it seems fair to say that refuting it will require a breakthrough in current algorithms for the D$\kap$S problem. There are also other supporting evidences for the validity of this hypothesis: please see the beginning of Appendix~\ref{sec:agnostic} for an additional discussion. As mentioned earlier, assuming this hypothesis, we can prove the tight running time lower bound for the bounded 1-ReLU Training problem:

\begin{theorem}[Tight Running Time Lower Bound for 1-ReLU Training] \label{thm:single-time-lower-bound-short}
Assuming Hypothesis~\ref{hyp:dks}, 
there is no algorithm that, for all given $\eps > 0$, can solve the bounded 1-ReLU training problem within an additive error $\eps$ in time $2^{o(1/\epsilon^2)}poly(n,m)$.
\end{theorem}

We remark that, akin to standard conventions in the area of fine-grained and parameterized complexity, all lower bounds are stated against algorithms that work for \emph{all} values of $\eps$ with the specified running time. Indeed, it is possible to significantly speed up the time bound $2^{O(1/\eps^2)} \poly(n, m)$ for extreme values of $\eps$; for instance, enumerating all possible $\bw$ over a $\Theta(\eps)$-net\footnote{Recall that an \emph{$\delta$-net} (also refer to as an \emph{$\delta$-cover}) of a set $S \subseteq \R^n$ is a set $T \subseteq \R^n$ such that, for every $x \in S$, there exists $y \in T$ where $\|x - y\|_2 \leq \delta$. It is well-known that, for any $\delta \in [0, 1]$, there is a $\delta$-net of the unit ball $\cB^n$ of size $(3/\delta)^n$ and that it can be found in $(3/\delta)^{O(n)}$ time.} of $\cB^n$ gives an algorithm that runs in time $O(1/\eps)^{O(n)} \poly(m)$, which is asymptotically smaller than $2^{O(1/\eps^2)} \poly(n, m)$ when $\eps = o\left(\frac{1}{\sqrt{n \log n}}\right)$. Nonetheless, our lower bounds can be extended to include a large range of ``reasonable'' $\eps$. Further discussion on such an extension is provided before Section~\ref{subsec:related}.

An interesting consequence of Theorem~\ref{thm:single-time-lower-bound-short} is that it gives a separation between proper and improper agnostic learning of 1-ReLU. Specifically,~\cite{goel2016reliably} shows that improper agnostic learning of 1-ReLU can be done in $2^{O(1/\eps)} \poly(n)$ time, while Theorem~\ref{thm:single-time-lower-bound-short} rules out such a possibility for proper agnostic learning. (See Appendix~\ref{subsec:learning-relus} for the relation between learning and training.)

\paragraph{Training $k$-ReLU: The Realizable Case.}
An important special case of the $k$-ReLU Training problem is the realizable case, where there is an unknown $k$-ReLU that labels every training sample correctly. When $k = 1$, it is straightforward to see that the realizable case of 1-ReLU Training can be phrased as a linear program and hence can be solved in polynomial time. On the other hand, we show that, once $k > 1$, the problem becomes NP-hard:

\begin{theorem}[Hardness of Training $k$-ReLU in the Realizable Case] \label{thm:realizable-hardness-short}
For any constant $k \geq 2$, the $k$-ReLU training problem is NP-hard even in the realizable case.
\end{theorem}

Our result is in fact slightly stronger than stated above: specifically, we show that, when the samples can be realizable by a (non-negative) sum of $k$ ReLUs (i.e. $k$-ReLU when $\ba$ is the all-one vector), it is still NP-hard to find a $k$-ReLU that realizes the samples even if negative coefficients in $\ba$ are allowed. Furthermore, while we assume in this theorem that $k$ is a constant independent of $n$, one can also prove an analogous hardness result, when $k$ grows sufficiently slowly as a function of $n$. We refer the reader to Appendix~\ref{sec:neg} for more details. 


Observe that Theorem~\ref{thm:realizable-hardness-short} implies that efficient \emph{multiplicative} approximation for the $k$-ReLU Training problem is impossible (assuming P$\neq$NP) for $k \geq 2$. As a result, we once again turn to additive approximation. On this front, we can improve the running time of the algorithm in Theorem~\ref{thm:agnostic-training-short} when we assume that the samples are realizable, as stated below.

\begin{theorem}[Training Algorithm in the Realizable Case] \label{thm:realizable-training-short}
When the given samples are realizable by some $k$-ReLU, there is a (randomized) algorithm that can solve the bounded $k$-ReLU training problem to within any additive error $\eps > 0$ in time $2^{O((k^3/\eps) \log^3(k/\eps))}\poly(n, m)$.
\end{theorem}

Importantly, the dependency of $\epsilon$ in the exponent is $\tilde{O}(1/\eps)$, instead of $1/\eps^2$ that appeared in the non-realizable case (i.e. Theorems~\ref{thm:agnostic-training-short} and~\ref{thm:single-time-lower-bound-short}). We can also show that this dependency is tight (up to log factors), in the realizable case, under the Gap Exponential Time Hypothesis (Gap-ETH)~\cite{Dinur16,ManurangsiR17}, a standard complexity theoretic assumption in parameterized complexity (see e.g.~\cite{ChalermsookCKLM17}). Gap-ETH states that there exists $\delta > 0$ such that no $2^{o(n)}$-time algorithm can, given a CNF formula with $n$ Boolean variables, distinguish between (i) the case where the formula is satisfiable, and (ii) the case where any assignment violates at least $\delta$ fraction of the clauses. Our running time lower bound can be stated more formally as follows.

\begin{theorem}[Tight Running Time Lower Bound for the Realizable Case] \label{thm:realizable-time-lower-bound-short}
Assuming Gap-ETH, for any constant $k \geq 2$, there is no algorithm that, for all given $\eps > 0$, can solve the bounded $k$-ReLU training problem within an additive error $\eps$ in time $2^{o(1/\epsilon)}poly(n,m)$ even when the input samples are realizable by some $k$-ReLU.
\end{theorem}

 

\paragraph{Relation to Learning ReLUs.} $k$-ReLU Training is closely related to the problem of \emph{proper learning} of $k$-ReLU. In fact, an algorithm for the latter also solves the former. Hence, our hardness results immediately implies hardness of proper learning of $k$-ReLU as well. Furthermore, our algorithm also works for the learning problem. Please refer to Section~\ref{subsec:learning-relus} for more details.

\paragraph{Stronger Quantifier in Running Time Lower Bounds.}
As stated earlier, our running time lower bounds in Theorems~\ref{thm:single-time-lower-bound-short} and~\ref{thm:realizable-time-lower-bound-short} hold only against algorithms that work \emph{for all} $\eps > 0$. A natural question is whether one can prove lower bounds against algorithms that work only \emph{for some} ``reasonable'' values of $\eps$. As explained in more detail below, we can quite easily also get a lower bound with this latter (stronger) quantifier, for any ``reasonable'' value of $\eps$.

First, our lower bounds in Theorems~\ref{thm:single-time-lower-bound-short} and~\ref{thm:realizable-time-lower-bound-short} both apply in the regime where the lower bounds themselves are $2^{\Theta(n)}$; in other words, $\eps = \Theta(1/\sqrt{n})$ in Theorem~\ref{thm:single-time-lower-bound-short} and $\eps = \Theta(1/n)$ in Theorem~\ref{thm:realizable-time-lower-bound-short}. These are essentially the smallest possible value of $\eps$ for which the lower bounds in Theorems~\ref{thm:single-time-lower-bound-short} and~\ref{thm:realizable-time-lower-bound-short} can hold, because the aforementioned algorithm that enumerates over an $\eps$-net of $\cB^n$ solves the problem in time $O(1 / \eps)^{O(n)} \poly(n)$. On the other hand, for smaller values of $\eps$, we can get a running time lower bound easily by ``padding'' the dimension by ``dummy'' coordinates that are always zero. For instance, if we start with $\eps = \Theta(1/\sqrt{n})$, then we may pad the instance to say $n' = n^2$ dimensions, resulting in the relationship $\eps = \Theta(1/\sqrt[4]{n'})$. To summarize, this simple padding technique immediately gives the following stronger quantifier version of Theorem~\ref{thm:single-time-lower-bound-short}:
\begin{theorem} \label{thm:single-time-padded}
For any non-increasing and efficiently computable\footnote{That is, we assume that computing $\eps(n)$ can be done in time $\poly(n)$ for any $n \in \N$.} function $\eps: \N \to \R^+$ such that $\omega(\sqrt{\log n}) \leq \frac{1}{\eps(n)} \leq o(\sqrt{n})$,
assuming Hypothesis~\ref{hyp:dks}, 
there is no algorithm that can solve the bounded 1-ReLU training problem within an additive error $\eps(n)$ in time $2^{o(1/\eps(n)^2)}poly(n,m)$.
\end{theorem}

Notice that the constraint $\omega(\sqrt{\log n}) \leq \frac{1}{\eps(n)}$ is also essentially necessary, because for $\eps > \sqrt{\frac{\log \log n}{\log n}}$ our algorithm (Theorem~\ref{thm:agnostic-training-short}) already runs in polynomial time.
A strong quantifier version of Theorem~\ref{thm:realizable-time-lower-bound-short} similar to above can be shown as well (but with $\omega(\log n) \leq \frac{1}{\eps(n)} \leq o(n)$). We omit the full (straightforward) proof via padding of Theorem~\ref{thm:single-time-padded}; interested readers may refer to the proof of Lemma 3.4 of~\cite{DKM-arxiv} which employs the same padding technique.


\subsection{Independent and concurrent work}
There have been several concurrent and independent works to ours that we mention here. We remark that the techniques in these works are markedly different than the ones in this paper.  For a single ReLU,~\cite{dey2018approximation} proved that the 1-ReLU Training problem is NP-hard.  With respect to two ReLUs,~\cite{Woodruf} showed that finding weights minimizing the 
squared error of a $2$-ReLU is NP-hard, even in the realizable case. The work of ~\cite{BDL18} considered the problem of training a network with a slightly different architecture, in which there are two ReLUs in the first hidden layer and the final output gate is also a ReLU (instead of a sum gate as in our case); they showed that, for such networks with three ReLUs (two in the hidden layer, one in the output layer), the training problem is NP-hard even for the realizable case. As a result of having an output gate computing a ReLU, our NP-hardness result (regarding training a sum of two ReLUs) does not imply their result and their hardness result does not imply our hardness result for training a sum of $2$ ReLUs. 

\subsection{Related work}
\label{subsec:related}

The computational aspects of training and learning neural networks has been extensively studied. Due to this, we only focus on those directly related to our results. 

We are not aware of a previous work showing that the general $k$-ReLU training problem is NP-hard for $k>2$, nor are we aware of previous results regarding the hardness of {\em approximating} the squared error of a single ReLU. The $k=2$ case and the $k>2$ case seem to require different ideas and indeed our proof technique for Theorem~\ref{thm:realizable-time-lower-bound-short} is different than those of~\cite{BDL18,Woodruf}. Moreover, the question of generalizing the NP-hardness result from $k=2$ to $k>2$ is mentioned explicitly in~\cite{BDL18}. Finally, we remark that neither~\cite{dey2018approximation} nor~\cite{Woodruf} provides explicit running time lower bounds in terms of $1/\epsilon$ for the problem of training $k$ ReLUs within an additive error of $\epsilon$. To the best of our knowledge, our work is the first to obtain such lower bounds.

~\cite{vu1998infeasibility} has proven that finding weights minimizing the squared error of a $k$-ReLU is NP-hard when $\ba$ is the all-one vector (or alternatively, when all the coefficients of the units are restricted to be positive) for every $k \geq 2$.

Some sources (e.g.~\cite{arora2018understanding,bach2017breaking}) attribute (either implicitly or explicitly) the NP-hardness of the $k$-ReLU Training problem to \cite{blum1989training}, who consider training a neural network with \emph{threshold units}. However, it is unclear (to us) how to derive the NP-hardness of training ReLUs from the hardness results of \cite{blum1989training}. Several NP-hardness results for training neural networks with architectures differing from the fully connected architecture considered here are known.
For example, in~\cite{brutzkus2017globally}, the training problem is shown to be hard for a depth-2 \emph{convolutional network} with (at least two) \emph{non-overlapping} patches. To the best of our knowledge, these architectural differences render those previous results inapplicable for deriving the hardness results regarding the networks considered in this work.

Several papers have studied a slightly different setting of \emph{improper learning} of neural networks. An example is~\cite{livni2014computational} who show that improper learning of depth-2 networks of $\omega(1)$ ReLUs is hard, assuming certain average case assumptions. More recently, \cite{goel2016reliably} show that even for a single ReLU, when $|\langle \bw, \bx\rangle|$ tends to infinity with $n$, learning $[\langle \bw, \bx\rangle]_+$ improperly in time
$g(\epsilon) \cdot poly(n)$ is unlikely as it will result in an efficient algorithm for the problem of learning sparse parities with noise which is believed to be intractable. These hardness results for improper learning do imply hardness for the corresponding training problems. Nonetheless, it should be noted that the fact that these results have to rely on assumptions other than P $\ne$ NP is not a coincidence: it is known that basing hardness of improper learning on P $\ne$ NP alone will result in a collapse of the Polynomial Hierarchy~\cite{ApplebaumBX08}.

On the algorithmic side, Arora et al.~\cite{arora2018understanding} provide a simple and elegant algorithm that exactly solves the ReLU training problem in polynomial time assuming the dimension $n$ of the data points is an absolute constant; Arora et al.'s algorithm is for the networks we consider, and it has since been also extended to other types of networks~\cite{BDL18}.
Additionally, there have also been works on (agnostic) learning algorithms for ReLUs. Specifically, Goel et al. \cite{goel2016reliably} consider the bounded norm setting where the inputs
to the ReLUs as well as the weight vectors of the units have norms at most $1$. For this setting, building on kernel methods and tools from approximation theory, they show how to \emph{improperly} learn a single $n$-variable ReLU up to an additive error of $\epsilon$ in time $2^{O(1/\epsilon)} \cdot poly(n)$. Their result generalizes to depth-2 ReLUs with $k$ units with running time of $2^{O(\sqrt{k}/\epsilon)} \cdot poly(n)$ assuming the coefficient vector $\ba$ has norm at most $1$. The algorithm they provide is quite general: it works for arbitrary distribution over input-output pairs, for $\epsilon$
that can be small as $1/\log n$ and also for the reliable setting. 

A limitation of our hardness results is that they consider ''pathological" training data sets that are specifically constructed to encode intractable combinatorial optimization problems. Several works in literature have tried to overcome this issue by considering the training/learning problems on more ``benign'' data distributions, such as log-concave distributions or those with Gaussian marginals. On this front, both algorithms and lower bounds have been shown for depth-2 networks~\cite{song2017complexity,Woodruf,goel2019time}.

Using insights from the study of exponential time algorithms towards understanding the complexity of machine learning problems as is done
in this work is receiving attention lately~\cite{servedio2017circuit,diakonikolas2019nearly,simonov2019refined}.

\subsection{Organization of the Paper}

In the remainder of the main body of this paper, we provide high-level overviews of our proofs (Section~\ref{sec:overview}) and discuss several potential research directions (Section~\ref{sec:open}). The appendix contains all the details of our proofs and is organized as follows. Appendix~\ref{sec:prelim} contains several additional notations that will be used throughout the proofs. In Appendix~\ref{sec:single-hardness}, we prove the NP-hardness of 1-ReLU Training (Theorem~\ref{thm:single-short}). We then prove the running time lower bound for the problem in Appendix~\ref{sec:agnostic}. In Appendix~\ref{sec:hardness-coloring}, we consider the problem of training (non-negative) sum of $k$ ReLUs, and prove hardness for the problem. We then use these hardness to prove our NP-hardness and running time lower bound of the $k$-ReLU Training problem (Theorems~\ref{thm:realizable-hardness-short},~\ref{thm:realizable-time-lower-bound-short}) in Appendix~\ref{sec:neg}. Finally, our algorithms are presented in Appendix~\ref{sec:learning}.

\section{Proof Overview}
\label{sec:overview}

Below we provide the informal overviews of our proofs and intuition behind them. All full proofs can be found in the appendix.

\paragraph{NP-Hardness of Training 1-ReLU.}
Our reduction is from the (NP-hard) Set Cover problem, in which we are given subsets $T_1, \dots, T_M$ of a universe $U$, and the goal is to select as few of these subsets as possible whose union covers the entire universe $U$. We reduce this to the problem of 1-ReLU Training, where the dimension $n$ is equal to $M$. We think of each coordinate of $\bw$ as an unknown (i.e. variable); specifically, the desired solution will have $w_i = -1$ iff $T_i$ is picked and 0 otherwise. From this perspective, adding a labelled sample $(\bx, y)$ is the same as adding a ``constraint'' $[\bw \cdot \bx]_+ = y$. There are two types of constraints we will add:
\begin{itemize}
\item (Element Constraint) For each $u \in U$, we add a constraint of the form $\left[1 + \sum_{T_i \ni u} w_i\right]_+ = 0$. The point is that such a constraint is satisfied when $u$ is covered by the selected subsets.
\item (Subset Constraint) For each $i \in [M]$, we add a constraint of the form $\left[\gamma + w_i\right]_+ = \gamma$ for some small $\gamma > 0$. This constraint will be violated for any selected subset.
\end{itemize}
By balancing the weights (i.e. number of copies) of each constraint carefully, we can ensure that the element constriants are never unsatisfied, and that the goal is ultimately to violate as few subset constraints as possible, which is equivalent to trying to pick as few subsets as possible that can fully cover $U$. This completes the high-level overview of our reduction.

We remark that there is a subtle point here because we cannot directly have a constant such that 1 or $\gamma$ in the constraints themselves. Rather, we need to have ``constraint coordinate'' and adding the constants through this coordinate. This will also be done in the other reductions presented below, and we will not mention this again.

The outlined proof, together with the $\Theta(\log |U|)$ inapproximability of Set Cover~\cite{LundY94,Feige98}, already gives a hardness of approximation of a multiplicative factor of $\Theta(\log(nm))$ for the 1-ReLU Training problem. To further improve this inapproximability ratio to $(nm)^{1/\poly \log \log (nm)}$, we reduce from the \emph{Minimum Monotone Circuit Satisfiability (MMCS)} problem, which is a generalization of Set Cover. In MMCS, we are given a monotone circuit and the goal is to set as few input wires to true as possible under the condition that the circuit's output must be true. Strong inapproximability results for MMCS are known (e.g.~\cite{DHK15}). Our reduction from MMCS proceeds in a similar manner as that of the Set Cover reduction above. Roughly speaking, the modification is that each unknown is now whether each wire is set/evaluated to true, whereas the constraints are now to ensure that the evaluation at each gate is correct and that the output is true.

\paragraph{Tight Running Time Hardness of 1-ReLU Training.}
We now move on to the proof overview of the tight running time lower bound for 1-ReLU Training. Recall that we will be reducing from the Densest $\kap$-Subgraph (D$\kap$S) problem, in which we are given a graph $G = (V, E)$ and $\kap \in \N$. The goal is to find a set of $\kap$ vertices that induces the maximum number of edges.

To motivate our construction, a simple combination of dimensionality reduction and $\delta$-net can in fact find a ReLU that point-wise approximates the optimal ReLU to within an additive factor of $\delta$ in time $2^{\tilde{O}(1/\delta^2)} \poly(n)$. That is, if the ReLU that achieves the optimal error has weight vector $\bw^*$, then we can find a weight vector $\bw$ such that $\left|[\bw \cdot \bx]_+ - [\bw^* \cdot \bx]_+\right| \leq \delta$ for all input samples $(\bx, y)$ in time\footnote{We assume throughout that $m = \poly(1/\delta)$, which is w.l.o.g. due to standard generalization bounds. See Section~\ref{sec:learning}.} $2^{\tilde{O}(1/\delta^2)} \poly(n)$.

Indeed, this is an explanation why, in the realizable case, we can get $\varepsilon$ squared error in $2^{\tilde{O}(1/\delta)} \poly(n)$ time by simply picking $\delta = \sqrt{\varepsilon}$. Now, since we need our hardness here (for the non-realizable case) to hold with stronger running time lower bound of $2^{\Theta(1/\varepsilon^2)} \poly(n)$, we have to make sure that whenever $\delta \gg \varepsilon$, the aforementioned point-wise approximation of $\delta$ is not sufficient to get an error of $\varepsilon$. Suppose that, for an input labelled sample $(\bx, y)$, the optimal ReLU outputs $y'$ and our approximation outputs $y''$ (where $|y'' - y'| \leq \delta$). Notice that the difference in the square error between the two for this sample is only at most $O((y' - y) \delta) + \delta^2$. Now, if we want this quantity to be at least $\varepsilon$ for any $\delta \geq \Omega(\varepsilon)$, then it must be that $|y' - y| = \Omega(1)$. In other words, we have to make our samples so that even the optimal ReLU is ``wrong'' by $\Omega(1)$ additive factor (on average); this indeed means that, if the ReLU we find is ``more wrong'' by an additive factor of $\Theta(\varepsilon)$, then the increase in the average squared error would be $\Omega(\varepsilon)$ as desired.

With the observation in the previous paragraph in mind, we will now provide a rough description of our gadget; they will all be formalized later in the proof of Lemma~\ref{lem:agn-reduction}. Given a D$\kap$S instance $(G = (V, E), \kap)$, our samples will have $|V|$ dimensions, one corresponding to each vertex. In the YES case where there is $T \subseteq V$ of size $\kap$ that induces many edges, we aim to have our ReLU weight assigning $\frac{1}{\sqrt{\kap}}$ to all coordinates corresponding to vertices in $T$, and zero to all other coordinates. To enforce this, we first add a sample for every vertex $v \in V$ that corresponds to the constraint
\begin{align*}
\left[\bw_v - \frac{1}{2\sqrt{\kap}}\right]_+ = 1.
\end{align*}
We refer to these as the \emph{cardinality constraints}.
While this may look peculiar at first glance, the effect is that it ensures that roughly speaking $\bw$ has $\kap$ coordinates that are ``approximately'' $\frac{1}{\sqrt{\kap}}$ and the remaining coordinates are ``small''. To see that this is the case, observe that the average mean squared error here is $1 - \frac{2}{|V|} \sum_{v \in V} \left[\bw_v - \frac{1}{2\sqrt{\kap}}\right]_+ + \frac{1}{|V|} \sum_{v \in V} \left[\bw_v - \frac{1}{2\sqrt{\kap}}\right]_+^2$. The last term is small and may be neglected. Hence, we essentially have to maximize $\sum_{v \in V} \left[\bw_v - \frac{1}{2\sqrt{\kap}}\right]_+$. This term is indeed maximized when $\bw$ has $\kap$ coordinates equal to $\frac{1}{\sqrt{\kap}}$, and zeros in the remaining coordinates. Notice here that this also fits with our intuition from the previous paragraph: even in the optimal ReLU, the value out put by the value (which is either 0 or $\frac{1}{2\sqrt{\kap}}$) is $\Omega(1)$ away from the input label of the sample (i.e. 1).

So far, the cardinality constraints have ensured that $\bw$ ``represents'' a set $T \subseteq V$ of size roughly $\kap$. However, we have not used the fact that $T$ contains many edges at all. Thus, for every edge $e = \{u, v\} \in E$, we also add the example corresponding to the following constraint to our distribution:
\begin{align*}
\frac{1}{2}\left[\bw_u + \bw_v - \frac{1.75}{\sqrt{\kap}}\right]_+ = 1.
\end{align*}
We call these the \emph{edge constraints}. The point here is that, if $e$ is not an induced edge in $T$, then the output of the ReLU will be zero. On the other hand, if $e$ is an edge in $T$, then the output of the ReLU will be $\frac{0.25}{\sqrt{\kap}}$. Hence, the more edges $T$ induces, the smaller the error.

By carefully selecting weights (i.e. number of copies) of each sample, one can indeed show that the average square error incurred in the completeness and soundness case of Hypothesis~\ref{hyp:dks} differs by $\varepsilon = \Omega\left(\frac{1}{\sqrt{|V|}}\right)$. Hence, if we can solve the 1-ReLU Training problem to within an additive error of $\eps$ in time $2^{o(1/\eps^2)}\poly(n, m)$, we can also solve the problem in Hypothesis~\ref{hyp:dks} in time $2^{o(|V|)}$, which breaks the hypothesis.


\paragraph{Hardness of Training $k$-ReLU in the Realizable Case.}
We next consider the problems of Training $k$-ReLU for $k \geq 2$ in the realizable case. Both the NP-hardness result (Theorem~\ref{thm:realizable-hardness-short}) and the tight running time lower bound (Theorem~\ref{thm:realizable-time-lower-bound-short}) employ similar reductions. 
These reductions proceed in two steps. First is to reduce from the NP-hard $k$-coloring problem to the problem of training \emph{non-negative sum of $k$ ReLUs}, in which we fix the coefficient vector $\ba$ to be the all-one vector and only seeks to find $\bw^1, \dots, \bw^k$ that minimizes the squared error. Then, in the second step, we reduce this to the original problem of $k$-ReLU Training (where the coefficient vector $\ba$ can be negative). 


\emph{Step I: From $k$-Coloring to Training Sum of $k$ ReLUs.}
The NP-hardness of Sum of $k$ ReLUs Training in fact follows directly from a reduction of~\cite{vu1998infeasibility}. We will now sketch Vu's reduction, since it will be helpful in our subsequent discussions 
below.
Vu's reduction starts from the $k$-coloring problem, in which we are given a hypergraph $G = (V, E)$ and the goal is to determine whether there is a proper $k$-coloring\footnote{A proper $k$-coloring is a mapping $\chi: V \to [k]$ such that no hyperedge is monochromatic, or equivalently $|\chi(e)| > 1$ for all $e \in E$.} of the hypergraph.
Given an instance $G = (V, E)$ of $k$-coloring, the number of dimensions in the training problem will be $n = |V|$ where we associate each dimension with a vertex. Notice that now we have $k$ unknowns associated to each vertex $v$: $w^1_v, \dots, w^k_v$. In the desired solution, these variables will tell us which color $v$ is assigned to: specifically, $w^i_v > 0$ iff $v$ is colored $i$ and $w^i_v \leq 0$ otherwise.

Adding a labelled sample $(\bx, y)$ is the same as adding a ``constraint'' $[\bw^1 \cdot \bx]_+ + \cdots + [\bw^k \cdot \bx]_+ = y$. There are two types of constraints we will add:
\begin{itemize}
\item (Vertex Constraint) For every vertex $v \in V$, we add a constraint\footnote{This constraint corresponds to $\bx$ being the $v$-th vector in the standard basis and $y = 1$.} $[w^1_v]_+ + \cdots + [w^k_v]_+ = 1$. This constraint ensures that, for every $v \in V$, we must have $w_v^{i_v} > 0$ for at least one $i_v \in [k]$, meaning that the vertex $v$ is assigned at least one color.
\item (Hyperedge Constraint) For every hyperedge $e = \{v_1, \dots, v_{\ell}\} \in E$, we add a constraint\footnote{This constraint corresponds to $\bx$ being the indicator vector of $e$ and $y = 0$.} $[w^1_{v_1} + \cdots + w^1_{v_\ell}]_+ + \cdots + [w^k_{v_1} + \cdots + w^k_{v_\ell}]_+ = 0$. This ensures that the hyperedge $e$ is not monochromatic. Otherwise, we have $i_{v_1} = \cdots = i_{v_\ell}$ meaning that $w^{i_{v_1}}_{v_1} + \cdots + w^{i_{v_1}}_{v_\ell} > 0$, which violates the hyperedge constraint.
\end{itemize}

This finishes our summary of Vu's reduction, which gives the NP-hardness of training a (non-negative) sum of $k$ ReLUs.


\emph{Step II: Handling Negative Coefficients.}
The argument above, especially for the hyperedge constraints, relies on the fact that the coefficient vector $\ba$ is the all-one vector. In other words, even if the input hypergraph is not $k$-coloring, it is still possible that there is a $k$-ReLU (possibly negative weight vector $\ba$) that realizes the samples. Hence, the reduction above does not yet work for our original problem of $k$-ReLU Training. To handle this issue, we use an additional gadget which is simply a set of labelled samples with the following properties: these samples can be realized by a $k$-ReLU only when the weight vectors $\ba$ is the all-one vector. Essentially speaking, by adding these samples also to our sample set, we have forced $\ba$ to be the all-one vector, at which point we restrict ourselves back to the case of (non-negative) sum of $k$ ReLUs and we can use the hard instance from the above reduction from $k$-coloring. These are the main ideas of the proof of Theorem~\ref{thm:realizable-hardness-short}.

\emph{Tight Running Time Lower Bound.}
As stated earlier, the tight running time lower bound for the bounded $k$-ReLU Training problem (Theorem~\ref{thm:realizable-time-lower-bound-short}) follows from a similar reduction, except that we now have to (1) carefully select the number of copies of each sample and (2) scale the labels $y_i$'s down so that the norm of each of $\bw^1, \dots, \bw^k$ is at most one. Roughly speaking, this means that the labels for the vertex constraints become $\Theta(1/\sqrt{|V|})$ instead of 1 as before. In other words, each violated constraint roughly contributes to $\Theta(1/|V|)$ squared error. Since it is known (assuming Gap-ETH) that distinguishing between a $k$-colorable hypergraph and a hypergraph for which every $k$-coloring violates a constant fraction of the edges takes $2^{\Omega(|V|)}$ time (e.g.~\cite{Patrank94}), we can arrive at the conclusion that solving the bounded $k$-ReLU Training problem to within an additive squared error of $\eps = \Theta(1/|V|)$ must take $2^{\Omega(1/|V|)} = 2^{\Omega(1/\eps)}$ time as desired.

We remark here that, interestingly,~\cite{vu1998infeasibility} used the reduction from $k$-coloring only for the case of $k = 2$ units and employed an additional gadget to handle the case $k > 2$. To the best of our knowledge, this approach seems to decrease the resulting error $\eps$, which means that the running time lower bound is not of the form $2^{\Omega(1/\eps)}$. On the other hand, we argue the hardness directly from $k$-coloring for any constant $k \geq 2$. This, together with a careful selection of the number of copies of each sample, allows us to achieve the running time lower bound in Theorem~\ref{thm:realizable-time-lower-bound-short}.

\paragraph{Training and Learning Algorithms.}
Our $k$-ReLU training algorithm is based on the approach of \cite{arora2018understanding}. The main idea behind the algorithm is to iterate over all possible sign patterns (whether each ReLU is active or not) of the inputs and subsequently solve the so formed convex optimization for each fixed pattern. The best hypothesis over all different sign patterns is chosen as the the final hypothesis. It is not hard to see that the run-time for such an algorithm would be $2^{(m+1)k}poly(n)$ since there are $2^{mk}$ different sign patterns.

Using standard generalization bounds, one can show that the number of samples $m$ needed for the empirical loss to be $\eps$ close to the true loss is at most $O(k^4/\epsilon^2)$. Plugging this into the above algorithm gets us the desired running time ($2^{O(k^5/\epsilon^2)}poly(n,m)$ as in Theorem \ref{thm:agnostic-training-short}) for the agnostic setting. For the realizable setting, we use an improved generalization result of~\cite{SrebroST10}, which implies that $m = \tilde{O}(k^2 / \epsilon)$ suffices; plugging this into the above algorithm yields us Theorem~\ref{thm:realizable-training-short}.

\section{Conclusions and Open Questions}
\label{sec:open}

We have studied the computational complexity of training depth-2 networks with the ReLU activation function providing both NP-hardness results
and algorithms for training ReLU's. Along the way we have introduced and used a new hypothesis regarding the hardness of approximating the Denset $\kap$-Subgraph problem in subexponential time that may find applications in other settings. Our results provide a separation between proper and improper learning showing that for a single ReLU, proper learning is likely to be harder than improper learning. Our hardness results regarding properly learning shallow networks suggest that improperly learning such networks (for example, learning overparametrized networks whose number of units far exceeds the dimension of the labeled vectors~\cite{allen2019learning,du2018gradient}) might be necessary to allow for tractable learning problems. 

We stress here that our hardness results apply to minimizing the population loss\footnote{The population loss is the expected square loss with respect to the distribution of data points.} as well, since one may simply create an instance where the population is just the training data.
Furthermore, the standard procedure for training neural networks is to perform ERM which is essentially minimizing the training loss. In fact, a bulk of theoretical work in the field focuses on generalization error assuming training error is small (often 0). Therefore, we believe it is a natural question to study the hardness of minimizing training loss.

Neural networks offer many choices (e.g., number of units, depth, choice of activation function, weight restrictions). Indicating which architectures are NP-hard to train can prove useful in guiding the search for a mathematical model of networks that can be trained efficiently. It should be remembered that our NP-hardness results are worst-case. Therefor they do not preclude efficient algorithms under additional distributional or structural assumptions~\cite{roughgarden2020beyond}.
Finally, as we focus on networks having significantly fewer units than data-points, the NP-hardness results reported here are not at odds with the ability to train neural networks in the overparmeterized regime where there are polynomial time algorithms that can fit the data with zero error~\cite{zhang2016understanding}. 

While we restrict our attention to algorithms for training networks with bounded weights, our exponential dependency of the running time on $k$ (the number of units) makes these algorithms impractical. It remains an interesting question whether the dependency of the running time on $k$ can be improved, or alternatively whether strong running time lower bounds can be shown in terms of $k$ (similar to what is done for $\eps$ in this work).

While we have focused on depth-2 networks, algorithms and lower bounds for deeper networks are of interest as well, especially given the multitude of their practical applications. It would be interesting to see whether the algorithms and hardness results extend to the setting of depth greater than $1$. An interesting concrete question here is whether training/learning becomes harder as the network becomes deeper. For instance, is it possible to prove running time lower bounds that grow with the depth of the network?

\section*{Acknowledgments} 
We would like to thank Amit Daniely, Amir Globerson, Meena Jagadeesan and Cameron Musco for interesting discussions.

\bibliographystyle{alpha}
\bibliography{relu}
\appendix

\section{Preliminaries and Notation}
\label{sec:prelim}

We use $\cB^n = \{\bx \in \mathbb{R}^n \mid \|\bx\|^2 \leq 1\}$ to denote the (closed) unit ball and $\cS^{n - 1} = \{\bx \in \mathbb{R}^n \mid \|\bx\|^2 = 1\}$ to denote the unit sphere in $n$ dimensions. Moreover, we use $\be_i$ to denote the $i$-th vector in the standard basis, i.e., $\be_i$ has its $i$-th coordinate being 1 and other coordinates being zeros.

For any $n, k \in \N$, $\bw^1, \dots, \bw^k \in \R^n$, $\ba \in \{-1, 1\}^k$ and distribution $\cD$ over $\R^n \times \R$, let
\begin{align*}
\cL(\bw^1, \dots, \bw^k, \ba; \cD) := \E_{(\bx, y) \sim \cD}\left[(\relu_{\bw^1, \dots, \bw^k, \ba}(\bx)-y_i)^2\right]
\end{align*}
to denote the expected squared loss with respect to $\cD$. We may write a sequence of labelled samples $S = ((\bx_i, y_i))_{i \in [m]}$ in place of $\cD$ to denote the expression when the distribution is uniform over $S$.

With this notation, the $k$-ReLU Training problem is: given a multiset of labelled samples $S = \{(\bx_i, y_i)\}_{i \in [m]}$ where $\bx_1, \dots, \bx_m \in \R^n, y_1, \dots, y_m \in \R$, find $\bw^1, \dots, \bw^k \in \R^n$ and $\ba \in \{-1, 1\}^k$ that minimizes $\cL(\bw^1, \dots, \bw^k, \ba; S)$. The bounded $k$-ReLU Training problem is similar except that $\bx_1, \dots, \bx_m \subseteq \cB^n, y_1, \dots, y_m \in [-k, k]$ and the minimization is over $\bw^1, \dots, \bw^k \in \cB^n, \ba \in \{-1, 1\}^k$. Additionally, we define the \emph{bounded sum of $k$-ReLU Training} problem to be the restriction of the bounded $k$-ReLU Training in which we only consider $\ba = (1, \dots, 1)$.

For brevity, when $\ba$ is the all-one vector (i.e. $a_1 = \cdots = a_k = 1$), we may drop $\ba$ from the notation and simply write $\cL(\bw; \cD) := \E_{(\bx, y) \sim \cD}[(\relu_{\bw^1, \dots, \bw^k, (1, \dots, 1)}(\bx) - y_i)^2]$.

Furthermore, we use $\bx \circ \bx'$ to denote the concatenation between vectors $\bx$ and $\bx'$, and $\bzero_p, \bone_p$ for $p \in \N$ to denote the $p$-dimensional all-zero vector and the $p$-dimensional all-one vector respectively.

\section{Hardness of training a single ReLU}
\label{sec:single-hardness}

In this section, we prove our hardness of 1-ReLU Training (Theorem~\ref{thm:single-short}). Specifically, our first result is the NP-hardness of the problem stated below, which proves the first part of Theorem~\ref{thm:single-short}.

\begin{theorem}\label{thm:single}
1-ReLU Training problem is NP-hard, even when the samples $\bx_i$ belong to $\{-1, 0,1\}^n$.
\end{theorem}

In all of our hardness reductions for 1-ReLU (both in this section and Section~\ref{sec:agnostic}), we will always consider non-negative labels $y_i$, which means that it is always better to pick $a_1$ to be $1$ and not $-1$. For convenience, we will assume this throughout and do not explicitly state that $a_1 = 1$.

\begin{proof}[Proof of Theorem~\ref{thm:single}]
We reduce the Set Cover problem to the 1-ReLU Training problem. Recall that, in the Set Cover problem, we are given a finite set $U$ along with a family $\cT = \{T_1,\ldots,T_M\}$ of $M$ subsets of $U$. Our goal is
to determine if one can choose $t$ subsets from $\cT$ whose union equals $U$.
Set Cover is well known to be NP-hard~\cite{Karp72}.

We consider a ReLU on $n = M + 2$ dimensions, where we view each coordinate of the (unknown) weight vector $\bw$ as a variable. Specifically, for each $T_i \in \cT$, we have a variable $w_{T_i}$. In addition, we have two dummy variable $w_1$ and $w_{\gamma}$.
We let $\gamma=0.01/M^2$.

We introduce the following labelled samples. First, for each $u \in U$, let $\bx_u$ be $n$-dimensional vector having $1$ for the coordinate corresponding to the dummy variable $w_1$, $-1$ in all coordinates that
correspond to a subset $T_i \in \cT$ containing $u$, and $0$ to all other coordinates. (In other words, $\bx_u = \be_1 + \sum_{T_i \ni u} \be_{T_i}$.) We label this vector by $y_u = 0$. This labelled sample corresponds to the constraint
\begin{equation}\label{equation:first}
\left[w_1 + \sum_{T_i \ni u} w_{T_i}\right]_+=0
\end{equation}
Second, for every $T_i \in \cT$ let $\bx_{T_i}$ be the $n$-dimensional vector having $-1$ in the $T_i$-th coordinate, $1$ in the coordinate corresponding to $w_{\gamma}$
and $0$ for all other coordinates. We label this vector by $\gamma$. (In other words, $\bx_{T_i} = \be_{\gamma} + \be_{T_i}$ and $y_{T_i} = \gamma$.) This corresponds to the constraint
\begin{equation}\label{equation:second}
[w_{\gamma} + w_{T_i}]_+=\gamma
\end{equation}
We add the $n$-dimensional vector having $1$ in the coordinate corresponding to $w_1$ and $0$ elsewhere. We label these vectors by $1$.
This corresponds to the constraint
\begin{equation}\label{equation:third}
[w_1]_+=1
\end{equation}
We add the $n$-dimensional vector having $1$ in the coordinate corresponding to $w_{\gamma}$ and $0$ elsewhere. We label these vectors by $\gamma$.
This corresponds to the constraint
\begin{equation}\label{equation:fourth}
[w_{\gamma}]_+=\gamma
\end{equation}
In summary, the sample multiset is $S = \{(\bx_u, y_u)\}_{u \in U} \cup \{(\bx_{T_i}, y_{T_i})\}_{T_i \in \cT} \cup \{(\be_1, 1)\} \cup \{(\be_\gamma, \gamma)\}$.
Clearly this reduction runs in polynomial time.
We now prove the correctness of this reduction.

(YES Case) Assume there is a cover of size $t$ for the set cover instance; suppose without loss of generality that this cover consists of the first $t$ subsets $T_1, \ldots, T_t$ in $\cT$. Assigning
$w_{T_1}=w_{T_2}=\ldots=w_{T_t}=-1, w_1=1$, $w_{\gamma}=\gamma$ and $0$ to all other variables results in an average squared error of $\frac{\gamma^2\cdot t}{|S|}$. This because exactly $t$ of the constraints from (\ref{equation:second}) are violated and each violated constraint contributes $\gamma^2$ to the squared error. All other constraints are satisfied.

(NO Case) Suppose contrapositively that there is a weight vector $\bw$ such that $\cL(\bw; S) \leq \frac{\gamma^2 \cdot t}{|S|}$. First, observe that $w_1 \geq 0.9$; otherwise, the squared error from~\eqref{equation:third} alone is more than $(0.1)^2 \geq \gamma^2 t$. Observe also that $w_\gamma \leq 0.2/M$; otherwise, the squared error from~\eqref{equation:fourth} must be more than $(0.2/M - \gamma)^2 \geq (0.1/M)^2 > \gamma^2 t$.

Our main observation is that the family $\cT_{< -w_\gamma} = \{T_i: w_{T_i}<-w_\gamma\}$ is a set cover. The reason is as follows: by the definition of $\gamma$ we have that $\sum_{T_i \in (\cT \setminus \cT_{< -w_\gamma}) } w_{T_i} \geq -w_\gamma \cdot M \geq -0.2$. As a result, if there is an element $u \in U$ that is not covered then the corresponding constraint (\ref{equation:first}) for $u$ will incur already a square error of at least $(0.7)^2 > \gamma^2 t$ (recall that $t$ is no larger than $m$). Thus, the observation follows.

The last step of the proof is to show that the family $\cT_{< -w_\gamma}$ contains at most $k$ subsets. To see that this is the case, observe that, for every $T_i \in \cT_{< -w_\gamma}$, we have $[w_\gamma + w_{T_i}]_+ = 0$, meaning that the corresponding constraint~\eqref{equation:second} incurs a squared error of $\gamma^2$. Since the total squared error is at most $\gamma^2 t$, we can immediately concludes that at most $t$ subsets belong to $\cT_{< -w_\gamma}$.

Thus, $\cT_{< -w_\gamma}$ is a set cover with at most $t$ subsets, which completes the NO case of the proof.
\end{proof}

Observe that in the hardness result above the set of samples is not realizable. This is not a coincidence as it is a simple result 
that when there are set of weights with zero error the training problem for a single ReLU is solvable in polynomial time via a simple application of linear programming. 

\subsection{Hardness of Approximating Minimum Training Error for a single ReLU}

The reduction above coupled with the fact that set cover is hard to approximate within a factor $O(\log |U|)$~\cite{LundY94,Feige98} in fact immediately implies that the problem of approximating the minimum squared error to within a \emph{multiplicative} factor of $O(\log(nm))$ is also hard. In this subsection, we will substantially improve this inapproximability ratio. Specifically, we will show that this problem is hard to approximate to within an almost polynomial (i.e. $(nm)^{1/\poly \log \log (nm)}$) factor, thereby proving the second part of Theorem~\ref{thm:single-short}:

\begin{theorem} \label{thm:single-inapprox}
1-ReLU Training problem is NP-hard to approximate to within a factor of $(nm)^{1/(\log \log (nm))^{O(1)}}$ where $n$ is the dimension of the samples and $m$ is the number of samples.
\end{theorem}

To prove Theorem~\ref{thm:single-inapprox}, we will reduce from the Minimum Monotone Circuit Satisfiability problem, which is formally defined below.

\begin{definition}
A \emph{monotone circuit} is a circuit where each gate is either an OR or and AND gate. We use $|C|$ to denote the number of wires in the circuit.
\end{definition}

\begin{definition}
In the \textsc{Minimum Monotone Circuit Satisfiability$_i$ (MMCS$_i$)} problem, we are given a monotone circuit of depth $i$, and the objective is to assign as few \textsc{True}s as possible to the input wires while ensuring that the circuit is satisfiable (i.e. output wire is evaluated to \textsc{True}).
\end{definition}

The hardness of approximating \mmcs\ has long been studied (e.g.~\cite{AlekhnovichBMP01,DinurS04}). The problem was known to be NP-hard to approximate to within a factor of $2^{\log^{1 - \varepsilon} |C|}$ for any constant $\varepsilon > 0$~\cite{DinurS04}. This has recently been improved to $|C|^{1/(\log \log |C|)^{O(1)}}$ by Dinur et al.~\cite{DHK15}\footnote{It should be noted that Dinur et al.~\cite{DHK15} in fact shows that there exists a PCP with $D = (\log \log n)^{O(1)}$ query over alphabet of size $n^{O(1/D)}$ with perfect completeness and soundness at most $1/n$. The result we use (Theorem~\ref{thm:mmcs}) follows from their result and from the reduction in Section 3 of~\cite{DinurS04} which shows how to reduce $D$-query PCP over alphabet $F$ with perfect completeness and soundness $s$ to an $\mmcs_3$ instance of size $F^D\poly(n)$ and gap $O(1/s)^{1/D}/D$. Plugging this in immediately implies the hardness we use.}.

\begin{theorem}[\cite{DHK15}] \label{thm:mmcs}
$\mmcs_3$ is NP-hard to approximate to within $|C|^{1/(\log \log |C|)^{O(1)}}$ factor.
\end{theorem}

The main result of this subsection is that, for any constant $\ell > 0$, there is a polynomial-time reduction from $\mmcs_\ell$ to the problem of minimizing the training error in single ReLU such that the optimum of the latter is proportional to the optimum of the former. From Theorem~\ref{thm:mmcs} above, this immediately implies Theorem~\ref{thm:single-inapprox}. The reduction is stated and proved below.


\begin{theorem}
For every $\ell > 0$, there is a polynomial-time reduction that takes in a depth-$\ell$ monotone circuit $C$ and produces samples $\{(\bx_i, y_i)\}_{i \in [m]}$ where $\bx_i \in \{0, 1\}^n$ such that the minimum squared training error\footnote{For convenience, we are using the \emph{total} squared error (i.e. $|S| \cdot \cL(\bw; S)$), not the average squared error (i.e. $\cL(\bw; S)$), in the theorem statement and its proof.} for these samples among all single ReLUs is exactly $\optmmcs(C) / (10|C|)^{2\ell + 2}$.
\end{theorem}

\begin{proof}
Let $\gamma := 1/(10|C|)^{\ell + 1}$. We consider a ReLU with $n = |C| + 1$ variables. For each wire $j$, we create a variable $w_j$. Additionally, we have a dummy variable $w_\gamma$. (Note that, in the desired solution, we want $w_j$ to be 1 iff the wire is evaluated to \textsc{True} and 0 otherwise, and $w_\gamma = \gamma$.)

{\bf Dummy Variable Constraint.} We add the following constraint
\begin{align} \label{eq:dummy-eps}
[w_\gamma]_+ = \gamma.
\end{align}

{\bf Input Wire Constraint.} For each input wire $i$, we add the constraint
\begin{align} \label{eq:input-wire}
[w_\gamma - w_i]_+ = \gamma.
\end{align}

{\bf Output Wire Constraint.} For the output wire $o$, we add the constraint
\begin{align} \label{eq:output-wire}
[w_o]_+ = 1.
\end{align}

{\bf OR Gate Constraint.} For each OR gate with input wires $i_1, \dots, i_k$ and output wire $j$, we add the constraint
\begin{align} \label{eq:or-gate}
[w_j - w_{i_1} - \cdots - w_{i_k}]_+ = 0.
\end{align}

{\bf AND Gate Constraint.} For each AND gate with input wires $i_1, \dots, i_k$ and output wire $j$, we add the following $k$ constraints:
\begin{align} \label{eq:and-gate}
[w_j - w_{i_1}]_+ = 0, \cdots, [w_j - w_{i_k}]_+ = 0.
\end{align}

We will now show that the minimum squared error possible is exactly $\optmmcs(C) \cdot \gamma^2$.

First, we will show that the error is at most $\optmmcs(C) \cdot \gamma^2$. Suppose that $\phi$ is an assignment to $C$ with $\optmmcs(C)$ \textsc{True}s that satisfies the circuit. We assign $w_\gamma = \gamma$, and, for each wire $j$, we assign $w_j$ to be 1 if and only if the wire $j$ is evaluated to be \textsc{True} on input $\phi$ and 0 otherwise. It is clear that every constraint is satisfied, except the input wire constraints~\eqref{eq:input-wire} for the wires that are assigned to \textsc{True} by $\phi$. There are exactly $\optmmcs(C)$ such wires, and each contributes $\gamma^2$ to the error; as a result, the training error of such weights is exactly $\optmmcs(C) \cdot \gamma^2$.

Next, we will show that the minimum squared training error has to be at least $\optmmcs(C) \cdot \gamma^2$. Suppose for the sake of contradiction that the minimum error $\delta$ is less than $\optmmcs(C) \cdot \gamma^2$.
Observe that, from $\optmmcs(C) \leq |C|$ and from our choice of $\gamma$, we have
\begin{align} \label{eq:delta-bound}
\delta < |C| \cdot \gamma^2 < 0.1
\end{align}

Consider an assignment $\phi$ that assigns each input wire $i$ to be \textsc{True} iff $w_i \geq w_\epsilon$. At the heart of this proof is the following proposition, which bounds the weight of every \textsc{False} wire.

\begin{proposition} \label{prop:height-bound}
For any wire $j$ at height $h$ that is evaluated to \textsc{False} on $\phi$, $w_j \leq (2|C|)^h \cdot (\gamma + \sqrt{\delta})$.
\end{proposition}

We note here that we define the height recursively by first letting the heights of all input wires be zero and then let the height of the output wire of each gate $G$ be one plus the maximum of the heights among all input wires of $G$.

\begin{proof}[Proof of Proposition~\ref{prop:height-bound}]
We will prove by induction on the height $h$.

{\bf Base Case.} Consider any input wire $i$ (of height 0) that is assigned \textsc{False} by $\phi$. By definition of $\phi$, we have $w_i < w_\gamma$. Note that $w_\gamma$ must be at most $\gamma + \sqrt{\delta}$, as otherwise the squared error incurred in~\eqref{eq:dummy-eps} is already more than $\delta$. Thus, we have $w_i \leq \gamma + \sqrt{\delta}$ as claimed.

{\bf Inductive Step.} Let $h \in \mathbb{N}$ and suppose that the statement holds for every \textsc{False} wire at height less than $h$. Let $j$ be any \textsc{False} at height $h$. Let us consider two cases:
\begin{itemize}
\item $j$ is an output of an OR gate. Let $i_1, \dots, i_k$ be the inputs of the gate. Since $j$ is evaluated to \textsc{False}, $i_1, \dots, i_k$ must all be evaluated to \textsc{False}. From our inductive hypothesis, we have $w_{i_1}, \dots, w_{i_k} \leq (2|C|)^{h - 1} \cdot (\gamma + \sqrt{\delta})$. Now, observe that $w_j$ can be at most $\sqrt{\delta} + w_{i_1} + \cdots + w_{i_k}$, as otherwise the squared error incurred in~\eqref{eq:or-gate} would be more than $\delta$. As a result, we have
\begin{align*}
w_j \leq \sqrt{\delta} + k \cdot (2|C|)^{h - 1} \cdot (\gamma + \sqrt{\delta}) = \sqrt{\delta} + |C| \cdot (2|C|)^{h - 1} \cdot (\gamma + \sqrt{\delta}) \leq (2|C|)^h \cdot(\gamma + \sqrt{\delta}).
\end{align*}
\item $j$ is an output of an AND gate. Let $i_1, \dots, i_k$ be the inputs of the gate. Since $j$ is evaluated to \textsc{False}, at least one of $i_1, \dots, i_k$ must all be evaluated to \textsc{False}. Let $i$ be one such wire. Observe that $w_j$ can be at most $\sqrt{\delta} + w_i$, as otherwise the squared error incurred in~\eqref{eq:and-gate} would be more than $\delta$. Hence, we have
\begin{align*}
w_j \leq \sqrt{\delta} + w_i \leq \sqrt{\delta} + (2|C|)^{h - 1} \cdot (\gamma + \sqrt{\delta}) \leq (2|C|)^h \cdot (\gamma + \sqrt{\delta}).
\end{align*}
where the second inequality comes from the inductive hypothesis.
\end{itemize}
In both cases, we have $w_j < (2|C|)^h  \cdot (\gamma + \sqrt{\delta})$, which concludes the proof of Proposition~\ref{prop:height-bound}.
\end{proof}

Now, consider the output wire $o$. We claim that $o$ must be evaluated to \textsc{True} on $\phi$. This is because, if $o$ is a \textsc{False} wire, then Proposition~\ref{prop:height-bound} ensures that $w_o$ is at most
\begin{align*}
(2|C|)^{\ell} \cdot (\gamma + \sqrt{\delta}) \stackrel{\eqref{eq:delta-bound}}{<} (2|C|)^\ell \cdot (\gamma + \sqrt{|C|} \cdot \gamma) \leq 0.1,
\end{align*}
where the second inequality comes from our choice of $\gamma$. This would mean that the squared error incurred in~\eqref{eq:output-wire} is at least $0.81 > \delta$. Thus, it must be that $\phi$ satisfies $C$.

Finally, observe that, since $\phi$ assigns each input wire $i$ to be \textsc{True} iff $w_i \geq w_\gamma$, each input wire that is assigned \textsc{True} incurs a squared error of $\gamma^2$ from~\eqref{eq:input-wire}. As a result, the number of input wires assigned \textsc{True} is at most $\frac{\delta}{\gamma^2} < \optmmcs(C)$, which is a contradiction as we argued above that $\phi$ satisfies $C$. This concludes our proof.
\end{proof}

\section{Running Time Lower Bound for 1-ReLU Training}\label{sec:agnostic}

In this section, we prove our nearly tight running time lower bound for bounded 1-ReLU Training (Theorem~\ref{thm:single-time-lower-bound-short}). Recall that our lower bound relies on the hypothesis that there is no $2^{o(N)}$-time algorithm that can approximate Densest $\kap$-Subgraph to within any (multiplicative) constant factor (Hypothesis~\ref{hyp:dks}). While this hypothesis might seem strong (especially given the fact that there is no known large constant factor inapproximability for D$\kap$S although we have such hardness under stronger assumptions, e.g.~\cite{AAMMW11,M17}), it should be noted that refuting it seems to be out of reach of known techniques. In particular, it is known that $o(N)$-level of the Sum-of-Squares Hierarchies do not give constant factor approximation for D$\kap$S even for bounded degree graphs~\cite{BCVGZ12,CMMV17,M-thesis}. Furthermore, these Sum-of-Squares lower bounds are proved via reductions from a certain family of random CSPs, whose Sum-of-Squares lower bounds are shown in~\cite{Sch08,Tul09}. This means that, if Hypothesis~\ref{hyp:dks} is false, then one can refute this family of sparse random CSPs in subexponential time. This would constitute an arguably surprising development in the area of refuting random CSPs, which has been extensively studied for decades (see~\cite{AOW15} and references therein).

We stress here that our lower bound in Theorem~\ref{thm:single-time-lower-bound-short} only holds if we are only allowed to consider a ReLU with weight vector $\bw$ within $\cB^n$ (i.e. having norm at most 1). If we modify the problem so that we are allowed to output a ReLU with arbitrary weight, then our lower bound in Theorem~\ref{thm:single-time-lower-bound-short} does not hold. It is an interesting open problem whether one can extend our lower bound to such modified problem as well, or whether a faster algorithm exists in that case.

The rest of this section is devoted to proving Theorem~\ref{thm:single-time-lower-bound-short}. The proof is easier to state if we consider a slight modification of Hypothesis~\ref{hyp:dks} where in the soundness we do not only guarantee that $\den_{\kap}(G)$ is small but also that $\den_{B\kap}(G)$ is small for some large constant $B$, as stated below.

\begin{hypothesis} \label{hyp:dks-modified}
For any constants $C, B \geq 1$, there exist $\delta = \delta(C, B) > 0$ and $d = d(C, B) \in \mathbb{N}$ such that the following holds. No $O(2^{\delta N})$-time algorithm can, given an instance $(G, \kap)$ of D$\kap$S where each vertex of $G$ has degree at most $d$ and an integer $\ell$, distinguish between the following two cases:
\begin{itemize}
\item (Completeness) $\den_{\kap}(G) \geq \ell$.
\item (Soundness) $\den_{B\kap}(G) < \ell/C$.
\end{itemize}
\end{hypothesis}

It turns out that the two hypotheses are actually equivalent:

\begin{proposition} \label{prop:eq-hypotheses}
Hypothesis~\ref{hyp:dks} and Hypothesis~\ref{hyp:dks-modified} are equivalent.
\end{proposition}

Since the proof of their equivalence is just a simple observation, we defer them to Appendix~\ref{app:eq-hypotheses}. With Hypothesis~\ref{hyp:dks-modified} in mind, we can now state (the properties of) the heart of our proof: the reduction from D$\kap$S to the problem of .
\begin{lemma} \label{lem:agn-reduction}
For some constants $C, B \geq 1$, there is a polynomial time algorithm that takes in an $N$-vertex graph $G$ with bounded degree $d$ and integers $\kap, \ell$, and produces a multiset of samples $S = \{(\bx_i, y_i)\}_{i \in [m]} \subseteq \cB^{n} \times  [0, 1]$ and two positive real numbers $\opt, \varepsilon \in \mathbb{R}^+$ such that
\begin{itemize}
\item (Completeness) If $\den_{\kap}(G) \geq \ell$, there $\bw \in \cB^n$ such that $\cL(\bw; S) \leq \opt$
\item (Soundness) If $\den_{B\kap}(G) < \ell/C$, then, for any $\bw \in \cB^n$, we have $\cL(\bw; S) > \opt + \varepsilon$.
\item (Error bound) $\varepsilon \geq \Omega_{d, C, B}\left(\frac{1}{\sqrt{N}}\right)$
\end{itemize}
\end{lemma}

By plugging in appropriate parameters, it is simple to see that Lemma~\ref{lem:agn-reduction} implies Theorem~\ref{thm:single-time-lower-bound-short}.

\begin{proof}[Proof of Theorem~\ref{thm:single-time-lower-bound-short}]
Suppose for the sake of contradiction that there is an $2^{o(1/\varepsilon^2)}\poly(n)$-time algorithm $\bA$ that solve the 1-ReLU Training problem to within an additive error of $\varepsilon$. We may solve the distinguishing problem in Hypothesis~\ref{hyp:dks-modified} as follows. Given an instance $(G, \kap, \ell)$, we first apply the reduction in Lemma~\ref{lem:agn-reduction} to produce a multiset of samples $S = \{(\bx_i, y_i)\}_{i \in [m]}$ where $\bx_i \in \cB^n$ and positive real numbers $\opt, \varepsilon = \Omega(1/\sqrt{N})$. We then run the algorithm $\bA$ with accuracy $\varepsilon$ to obtain a ReLU weight vector $\bw$ with norm at most one. By checking whether $\cL(\bw; S) \leq \opt + \varepsilon$, we have distinguished the two cases in Hypothesis~\ref{hyp:dks-modified}. Furthermore, our algorithm runs in time $2^{o(1/\varepsilon^2)} \poly(n) = 2^{o(n)} \poly(n)$. Hence, this violates Hypothesis~\ref{hyp:dks-modified}.

We conclude the proof by recalling that Hypotheses~\ref{hyp:dks} and~\ref{hyp:dks-modified} are equivalent by Proposition~\ref{prop:eq-hypotheses}.
\end{proof}

\subsection{Reducing Densest $k$-Subgraph to Agnostic Learning of ReLUs}

We proceed to the main technical contribution of this section, the proof of Lemma~\ref{lem:agn-reduction}. The proof closely follows the intuition given in Section~\ref{sec:overview}.

\begin{proof}[Proof of Lemma~\ref{lem:agn-reduction}]
We will give the reduction for $C, B = 1000$. Before we specify $S$, recall that we use $N$ and $M$ to denote the number of vertices and the number of edges of $G$ respectively. Furthermore, let us define several additional parameters that will be used throughout:
\begin{itemize}
\item Let $\delta = \frac{1}{2\sqrt{\kap}}$.
\item Let $\gamma = \frac{1}{1000d}$ and $\zeta = \frac{1}{10^{10}d^2}$
\item Let $\opt = (1 - \gamma)\zeta \cdot \left(1 - \frac{\kap \delta}{\sqrt{2} N} + \frac{\kap \delta^2}{8N}\right) + \gamma \zeta \cdot \left(1 - \frac{\ell\delta}{2\sqrt{2} M} + \frac{\ell\delta^2}{32M}\right)$.
\item Let $\varepsilon = \gamma \zeta \cdot \frac{\ell \delta}{4\sqrt{2} M} - (1 - \gamma)\zeta \cdot \frac{\kap \delta^2}{8N} - \gamma\zeta \cdot \frac{\ell \delta^2}{32M}$.
\end{itemize}
Now that, a priori, it may not be clear that $\varepsilon$ is even positive. However, note that both of the terms $(1 - \gamma)\zeta \cdot \frac{\kap \delta^2}{8N}$ and $\gamma\zeta \cdot \frac{\ell \delta^2}{32M}$ are $O_{d}(\frac{1}{N})$. However, $\gamma \zeta \cdot \frac{\ell \delta}{4\sqrt{2} M} = \Omega_d(\frac{\ell}{\sqrt{\kap}N})$. Now, notice that, we may assume w.l.o.g. that\footnote{Note that, in the non-trivial case, it is always simple to find $\kap$ vertices that induce $\lfloor \kap/2 \rfloor$ edges, by repeatedly adding one edge at a time. This bound is at least $\kap / 3$ for any $\kap \geq 2$.} $\ell \geq \kap/3$ and that\footnote{In particular, if $\kap = o(N)$, then the algorithm that enumerate all subsets of size $\kap$ already runs in time polynomial in $\binom{N}{\kap} = 2^{o(N)}$.} $\kap \geq \Omega(N)$. Hence, this positive term is at least $\Omega_d(\frac{1}{\sqrt{N}})$. In other words, for sufficiently large $N$, $\varepsilon$ is positive and furthermore $\varepsilon = \Omega_d(\frac{1}{\sqrt{N}})$ as desired.

We can now define the multiset of samples $S = \{(\bx_i, y_i)\}_{i \in [m]}$ where $\bx_i \in \cB^n$ as follows. First, we let $n = N + 1$; we associate each of the first $N$ coordinates by each vertex of $G$ and we name the last coordinate $*$.
\begin{itemize}
\item We create $(10^{20} \cdot d^3NM) \cdot (1 - \zeta)$ copies of the labelled sample $(\be_*, \frac{1}{\sqrt{2}})$ in $S$. This corresponds to the constraint
\begin{align*}
[\bw_*]_+ = \frac{1}{\sqrt{2}}.
\end{align*}
We refer to this as the \emph{constant constraint for $*$}.
\item For each vertex $v \in V$, we add $(10^{20} \cdot d^3NM) \cdot \frac{(1 - \gamma)\zeta}{N}$ copies of the sample $(\frac{1}{2}\left(\be_v - \delta \be_*\right), 1)$ to $S$. This corresponds to the constraint
\begin{align*}
\frac{1}{2} [\bw_v - \delta \bw_*]_+ = 1.
\end{align*}
This is referred to as the \emph{cardinality constraint for $v$}.
\item Finally, for each edge $e = \{u, v\} \in E$, we add $(10^{20} \cdot d^3NM) \cdot \frac{\gamma \zeta}{M}$ copies of the sample $(\frac{1}{2}\left(\bw_u + \bw_v - 3.5\delta \bw_*\right), 1)$ to $S$. This corresponds to the constraint
\begin{align*}
\frac{1}{2} [\bw_u + \bw_v - 3.5\delta \bw_*]_+ = 1.
\end{align*}
We refer to this as the \emph{edge constraint for $e$}.
\end{itemize}
For notational convenience, let us separate the average square error $\cL(\bw; S)$ into three parts, based on the type of constraints. More specifically, we let
\begin{align*}
\cL^*(\bw; S) &= (1 - \zeta) \cdot \left(\frac{1}{\sqrt{2}} - [\bw_*]_+\right)^2, \\
\cL^{\card}(\bw; S) &= \frac{(1 - \gamma)\zeta}{N} \cdot \left(\sum_{v \in V} \left(1 - \frac{1}{2} [\bw_v - \delta \bw_*]_+\right)^2\right), \text{ and } \\
\cL^{\edge}(\bw; S) &= \frac{\gamma \zeta}{M} \cdot \left(\sum_{\{u, v\} \in E}\left(1 - \frac{1}{2}[\bw_u + \bw_v - 3.5\delta \bw_*]_+\right)^2\right).
\end{align*}
By definition, we of course have $\cL(\bw; S) = \cL^*(\bw; S) + \cL^{\card}(\bw; S) + \cL^{\edge}(\bw; S)$. It will also be useful to expand out the term $\cL^{\card}(\bw; S)$ and $\cL^{\edge}(\bw; S)$ as follows:
\begin{align*}
\frac{\cL^{\card}(\bw; S)}{(1 - \gamma)\zeta/N} &= N - \left(\sum_{v \in V} [\bw_v - \delta \bw_*]_+\right) + \frac{1}{4} \left(\sum_{v \in V} [\bw_v - \delta \bw_*]_+^2\right), \\
\frac{\cL^{\edge}(\bw; S)}{\gamma\zeta/M} &= M - \left(\sum_{\{u, v\} \in E} [\bw_u + \bw_v - 3.5\delta\bw_*]_+\right) + \frac{1}{4} \left(\sum_{\{u, v\} \in E} [\bw_u + \bw_v - 3.5\delta\bw_*]_+^2\right).
\end{align*}

\paragraph{(Completeness)}
Suppose that there exists a set $T \subseteq V$ of size $k$ that induces at least $\ell$ edges. Then, we can set $\bw_* = \frac{1}{\sqrt{2}}$, $\bw_v$ to be $\delta \sqrt{2}$ iff $v \in T$ and zero otherwise. It is obvious to see that the $\|\bw\|_2 = 1$ as desired. Moreover, we have $\cL^{*}(\bw; S) = 0$,
\begin{align*}
\cL^{\card}(\bw; S) = (1 - \gamma)\zeta \cdot \left(1 - \frac{\kap \delta}{\sqrt{2} N} + \frac{\kap \delta^2}{8N}\right),
\end{align*}
and
\begin{align*}
\cL^{\edge}(\bw; S) \leq \gamma \zeta \cdot \left(1 - \frac{\ell\delta}{2\sqrt{2} M} + \frac{\ell\delta^2}{32M}\right).
\end{align*}
In total, we have $\cL(\bw; S) \leq \opt$ as desired.

\paragraph{(Soundness)}
Suppose for the sake of contradiction that $\den_{B\kap}(G) \leq \ell / C$ but there exists $\bw \in \cB^n$ such that $\cL(\bw; S) \leq \opt + \varepsilon$. Let $\lambda_1 = \delta \bw_*$ and $\lambda_2 = 2.5\delta \bw_*$. We partition the set of vertices $V$ into three sets:
\begin{itemize}
\item $V_{\geq \lambda_2} := \{v \in V \mid \bw_v \geq \lambda_2\}$.
\item $V_{(\lambda_1, \lambda_2)} := \{v \in V \mid \bw_v \in (\lambda_1, \lambda_2)\}$.
\item $V_{\leq \lambda_1} := \{v \in V \mid \bw_v \leq \lambda_1\}$.
\end{itemize}
From this point on, we will write each edge $\{u, v\}$ in $E$ as an ordered tuple $(u, v)$ such that $\bw_u \geq \bw_v$ (tie broken arbitrarily).
We can then partition the set of edges $E$ into three parts:
\begin{itemize}
\item $E_{\geq \lambda_2} := \{(u, v) \in E \mid u \in V_{\geq \lambda_2}\}$.
\item $E_{(\lambda_1, \lambda_2)} := \{(u, v) \in E \mid u, v \in V_{(\lambda_1, \lambda_2)} \}$.
\item $E_{\leq \lambda_1} := \{(u, v) \in E \mid v \in V_{\leq \lambda_1} \wedge u \notin V_{\geq \lambda_2}\}$.
\end{itemize}

Observe that
\begin{align} \label{eq:edge-ignore-sq}
\cL^{\edge}(\bw; S) &\geq \frac{\gamma \zeta}{M}\left(M - \sum_{(u, v) \in E} [\bw_u + \bw_v - 3.5\delta \bw_*]_+\right).
\end{align}
Let us now write $\sum_{(u, v) \in E} [\bw_u + \bw_v - 3.5\delta \bw_*]_+$ as
\begin{align*}
\sum_{(u, v) \in E_{\geq \lambda_2}} [\bw_u + \bw_v - 3.5\delta \bw_*]_+ + \sum_{(u, v) \in E_{(\lambda_1, \lambda_2)}} [\bw_u + \bw_v - 3.5\delta \bw_*]_+ + \sum_{(u, v) \in E_{\leq \lambda_1}} [\bw_u + \bw_v - 3.5\delta \bw_*]_+.
\end{align*}
The last summation $\sum_{(u, v) \in E_{\leq \lambda_1}} [\bw_u + \bw_v - 3.5\delta \bw_*]_+$ is simply zero because, for all $(u, v) \in E_{\leq \lambda_1}$, we have $\bw_u < \lambda_2$ and $\bw_v \leq \lambda_1$, meaning that $\bw_u + \bw_v - 3.5\delta \bw_* < 0$.

Let us now consider the second term $\sum_{(u, v) \in E_{(\lambda_1, \lambda_2)}} [\bw_u + \bw_v - 3.5\delta \bw_*]_+$. Observe that $\cL(\bw; S) \leq \opt + \varepsilon \leq 0.01$ implies that $\bw_* \geq \frac{1}{3}$, which means that $\lambda_1 \geq \frac{\delta}{3} = \frac{1}{6\sqrt{\kap}}$. Notice that $E_{(\lambda_1, \lambda_2)}$ is exactly the set of edges induced by $V_{(\lambda_1, \lambda_2)}$. Since $\lambda_1 \geq \frac{1}{6\sqrt{\kap}}$, we must have $|V_{(\lambda_1, \lambda_2)}| < 36\kap$ (because otherwise $\|\bw\|_2 > 1$). As a result, from the assumption that $\den_{B\kap}(G) \leq \ell / C$, we have $|E_{(\lambda_1, \lambda_2)}| < 0.001\ell$. Hence, we have
\begin{align*}
\sum_{(u, v) \in E_{(\lambda_1, \lambda_2)}} [\bw_u + \bw_v - 3.5\delta \bw_*]_+ < 0.001 \ell \cdot 1.5 \delta w_* < 0.01 \ell \cdot \delta.
\end{align*}

Finally, let us bound $\sum_{(u, v) \in E_{\geq \lambda_2}} [\bw_u + \bw_v - 3.5 \delta \bw_*]_+$ as follows:
\begin{align*}
\sum_{(u, v) \in E_{\geq \lambda_2}} [\bw_u + \bw_v - 3.5 \delta \bw_*]_+ &\leq \sum_{(u, v) \in E_{\geq \lambda_2}} \left([2 \bw_u - 3.5\delta \bw_*]_+\right) \\
&\leq \sum_{(u, v) \in E_{\geq \lambda_2}} 2[\bw_u - \delta \bw_*]_+ \\
&\leq 2d \sum_{u \in V_{\geq \lambda_2}} [\bw_u - \delta \bw_*]_+,
\end{align*}
where the last inequality follows from the fact that every vertex in graph $G$ has degree at most $d$.

Combining the above two inequalities, we have
\begin{align*}
\sum_{(u, v) \in E} [\bw_u + \bw_v - 3.5\delta \bw_*]_+ < 0.01 \ell \delta + 2d \sum_{u \in V_{\geq \lambda_2}} [\bw_u - \delta \bw_*]_+.
\end{align*}

Plugging the above inequality back into~\eqref{eq:edge-ignore-sq}, we arrive at
\begin{align} \label{eq:edge-bound-int}
\cL^{\edge}(\bw; S) > \frac{\gamma \zeta}{M}\left(M - 0.01 \ell \delta - 2d \sum_{u \in V_{\geq \lambda_2}} [\bw_u - \delta \bw_*]_+\right).
\end{align}

Observe also that
\begin{align}
\cL^{\card}(\bw; S) &\geq \frac{(1 - \gamma)\zeta}{N}\left(N - \sum_{v \in V}[\bw_v - \delta \bw_*]_+ \right) \nonumber \\
&=  \frac{(1 - \gamma)\zeta}{N}\left(N - \sum_{v \in V_{\geq \lambda_2}}[\bw_v - \delta \bw_*]_+ - \sum_{v \in V_{(\lambda_1, \lambda_2)}}[\bw_v - \delta \bw_*]_+ \right) \label{eq:card-bound-int}
\end{align}

By summing up~\eqref{eq:edge-bound-int} and~\eqref{eq:card-bound-int}, we have
\begin{align}
&\cL^{\edge}(\bw; S) + \cL^{\card}(\bw; S) \nonumber \\
&> \frac{\gamma \zeta}{M}\left(M - 0.01 \ell \delta - 2d \sum_{u \in V_{\geq \lambda_2}} [\bw_u - \delta \bw_*]_+\right) \nonumber \\
&\qquad + \frac{(1 - \gamma)\zeta}{N}\left(N - \sum_{v \in V_{\geq \lambda_2}}[\bw_v - \delta \bw_*]_+ - \sum_{v \in V_{(\lambda_1, \lambda_2)}}[\bw_v - \delta \bw_*]_+ \right) \nonumber \\
&\geq \frac{\gamma \zeta}{M}\left(M - 0.01 \ell \delta\right) + \frac{(1 - \gamma)\zeta}{N}\left(N - \left(1 + \frac{2\gamma d N}{(1 - \gamma)M}\right) \cdot \sum_{v \in V_{\geq \lambda_2}}[\bw_v - \delta \bw_*]_+ - \sum_{v \in V_{(\lambda_1, \lambda_2)}}[\bw_v - \delta \bw_*]_+ \right) \nonumber \\
&\geq \frac{\gamma \zeta}{M}\left(M - 0.01 \ell \delta\right) + \frac{(1 - \gamma)\zeta}{N}\left(N - 1.01 \sum_{v \in V_{\geq \lambda_2}}[\bw_v - \delta \bw_*]_+ - \sum_{v \in V_{(\lambda_1, \lambda_2)}}[\bw_v - \delta \bw_*]_+ \right), \label{eq:err-edge-and-card}
\end{align}
where in the last inequality we use the fact that $\frac{2\gamma}{(1 - \gamma)} < \frac{0.01}{d}$ which follows from our choice of $\gamma$.

Now, for each $v \in V_{(\lambda_1, \lambda_2)}$, the AM-GM inequality implies that
\begin{align*}
\bw_v^2 = (\delta \bw_* + (\bw_v - \delta \bw_*))^2 \geq 4\delta \bw_* (\bw_v - \delta \bw_*) = 4\delta \bw_* [\bw_v - \delta \bw_*]_+.
\end{align*}

Similarly, for each $v \in V_{\geq \lambda_2}$, the AM-GM inequality implies that
\begin{align*}
\bw_v^2 &= (1.25 \delta \bw_* + (\bw_v - 1.25 \delta \bw_*))^2 \\
&\geq 5\delta \bw_* (\bw_v - 1.25 \delta \bw_*) \\
&> 4.1\delta \bw_*(\bw_v - \delta \bw_*) \\
&\geq 4.1\delta \bw_* [\bw_v - \delta \bw_*]_+,
\end{align*}
where the second-to-last inequality follows from $\bw_v \geq \lambda_2 \bw_* = 2.5 \delta \bw_*$, which implies that $(\bw_v - 1.25 \delta \bw_*) \geq \frac{1.25}{1.5}(\bw_v - \delta \bw_*)$.

Plugging the above two inequalities back into~\eqref{eq:err-edge-and-card}, we get
\begin{align*}
\cL^{\edge}(\bw; S) + \cL^{\card}(\bw; S) &> \frac{\gamma \zeta}{M}\left(M - 0.01 \ell \cdot \delta\right) + \frac{(1 - \gamma)\zeta}{N}\left(N - \frac{1}{4\delta \bw_*} \sum_{v \in V} \bw_v^2\right) \\
&\geq \frac{\gamma \zeta}{M}\left(M - 0.01 \ell \cdot \delta\right) + \frac{(1 - \gamma)\zeta}{N}\left(N - \frac{1}{4\delta \bw_*} (1 - \bw_*^2)\right),
\end{align*}
where the second inequality follows from $\|\bw\|_2 \leq 1$.

Recall also that $\cL^*(\bw; S) = (1 - \zeta)\left(\frac{1}{\sqrt{2}} - \bw_*\right)^2$. Adding this to above, we have
\begin{align}
\cL(\bw; S) &> \frac{\gamma \zeta}{M}\left(M - 0.01 \ell \cdot \delta\right) + \frac{(1 - \gamma)\zeta}{N}\left(N - \frac{1}{4\delta \bw_*} (1 - \bw_*^2)\right) + (1 - \zeta)\left(\frac{1}{\sqrt{2}} - \bw_*\right)^2 \nonumber \\
&= \frac{\gamma \zeta}{M}\left(M - 0.01 \ell \cdot \delta\right) + (1 - \gamma)\zeta - \left(\frac{(1 - \gamma)\zeta}{N} \cdot \frac{1}{4\delta \bw_*} \cdot (1 - \bw_*^2) - (1 - \zeta)\left(\frac{1}{\sqrt{2}} - \bw_*\right)^2\right). \label{eq:err-bound-two-sum}
\end{align}
We will now bound the term
\begin{align*}
D(\bw_*) := \frac{(1 - \gamma)\zeta}{N} \cdot \frac{1}{4\delta \bw_*} \cdot (1 - \bw_*^2) - (1 - \zeta)\left(\frac{1}{\sqrt{2}} - \bw_*\right)^2.
\end{align*}
In particular, we will show that $D(\bw^*) < D(1/\sqrt{2}) + \frac{5\zeta^2}{N}$.

First, notice that, if $\bw_* \geq \frac{1}{\sqrt{2}}$, then we immediately have $D(\bw^*) \leq D(1/\sqrt{2})$, as the former is larger than the latter term-wise. Hence, we may only consider the case $\bw_* < \frac{1}{\sqrt{2}}$. Let $\varphi_* = \frac{1}{\sqrt{2}} - \bw_*$. We may write $D(\bw^*) - D(1/\sqrt{2})$ as
\begin{align*}
D(\bw_*) - D(1/\sqrt{2}) &= \frac{(1 - \gamma)\zeta}{4\delta N} \left(\frac{\sqrt{2} \varphi_*}{\bw_*} + \varphi_*\right) - (1 - \zeta) \varphi_*^2 \\
&= \varphi_* \left(\frac{(1 - \gamma)\zeta}{4\delta N}\left(\frac{\sqrt{2}}{\bw_*} + 1\right) - (1 - \zeta) \varphi_*\right)
\end{align*}
Recall that $\bw_* \geq \frac{1}{3}$. As a result, we must have
\begin{align*}
D(w^*) - D(1/\sqrt{2}) &\leq \varphi_* \left(\frac{2(1 - \gamma)\zeta}{\delta N}- (1 - \zeta) \varphi_*\right) \\
&= \frac{1}{1 - \zeta} \cdot \left((1 - \zeta)\varphi_*\right)\left(\frac{2(1 - \gamma)\zeta}{\delta N}- (1 - \zeta) \varphi_*\right) \\
(\text{AM-GM Inequality}) &\leq \frac{1}{1 - \zeta} \left(\frac{(1 - \gamma)\zeta}{\delta N}\right)^2 \\
&< \frac{5\zeta^2}{N},
\end{align*}
where the last inequality follows from $\zeta = 0.99$ and $\delta = \frac{1}{2\sqrt{\kap}} \geq \frac{1}{2\sqrt{N}}$. Thus, in both cases, we have $D(\bw_*) < D(1/\sqrt{2}) + \frac{5\zeta^2}{N}$. Plugging this back into~\eqref{eq:err-bound-two-sum}, we have
\begin{align*}
\cL(\bw; S) &> \frac{\gamma \zeta}{M}(M - 0.01 \ell \delta) + (1 - \gamma) \zeta - \frac{(1 - \gamma)\zeta}{4\sqrt{2} \delta N} - \frac{5\zeta^2}{N} \\
&> \gamma \zeta\left(1 - \frac{\ell \delta}{4\sqrt{2} M}\right) + (1 - \gamma) \zeta \left(1 - \frac{1}{4\sqrt{2} \delta N}\right) + \left(\frac{0.01\ell \delta \gamma \zeta}{M} - \frac{5\zeta^2}{N}\right) \\
&= OPT + \varepsilon + \left(\frac{0.01\ell \delta \gamma \zeta}{M} - \frac{5\zeta^2}{N}\right) \\
&\geq OPT + \varepsilon + \frac{0.01\zeta}{N} \left(\frac{\ell\delta \gamma}{d} - 500\zeta \right) \\
&\geq OPT + \varepsilon + \frac{0.01\zeta}{N} \left(\frac{(\kap/3) \cdot \frac{1}{2\sqrt{\kap}} \cdot \gamma}{d} - 500\zeta \right) \\
&\geq OPT + \varepsilon,
\end{align*}
where, in the second to last inequality, we assume w.l.o.g. that $\ell \geq \kap/3$ and the last inequality follows from our choice of $\zeta$ and $\gamma$. This is a contradiction.
\end{proof}

\subsection{Simplifying the Hypothesis: Proof of Proposition~\ref{prop:eq-hypotheses}}
\label{app:eq-hypotheses}

The proof is a simple ``trivial'' reduction; the key observation here is that $\den_{B\kap}(G)$ cannot be much larger than $\den_{\kap}(G)$. Hence, by picking the constant $C$ in Hypothesis~\ref{hyp:dks} to be sufficiently large, we can arrive at a hypothesis of the form stated in Hypothesis~\ref{hyp:dks-modified}.

\begin{proof}[Proof of Proposition~\ref{prop:eq-hypotheses}]
It is obvious that Hypothesis~\ref{hyp:dks-modified} implies Hypothesis~\ref{hyp:dks}, by simply plugging $B = 1$ into the former.

To prove the converse, for any $C, B > 1$, let $C' = $ and let $\delta = \delta(C')$ and $d = d(C')$ be as in Hypothesis~\ref{hyp:dks}. Now, we claim that, if $\den_{\kap}(G) < \ell / C'$ for any graph $G$ and any $\ell$, then $\den_{B\kap}(G) < \ell / C$. To see that this is the case, suppose contrapositively that $\den_{B\kap}(G) \geq \ell / C$, i.e., there exists $T \subseteq V$ of size $B\kap$ such that $|E(T)| > \ell / C$. Then, let us consider a random subset $T' \subseteq T$ of size $k$. We have
\begin{align*}
\E_{T'}[|E(T')] = |E(T)| \cdot \left(\frac{\kap(\kap - 1)}{B\kap(B\kap - 1)}\right) \geq \frac{\ell}{C} \cdot \frac{1}{2B^2} = \frac{\ell}{C'}.
\end{align*}
where we assume w.l.o.g. that $\kap \geq 2$ in the inequality. This indeed implies that $\den_{\kap}(G) \geq \ell / C'$.

The previous paragraph means that, if we can distinguish the two cases in Hypothesis~\ref{hyp:dks-modified} (with constants $C, B$), then we can also distinguish the two cases in Hypothesis~\ref{hyp:dks} (with constant $C'$). As a result, if the former cannot be done in $O(2^{\delta N})$ time, then nor does the latter. In other words, Hypothesis~\ref{hyp:dks} implies Hypothesis~\ref{hyp:dks-modified}.
\end{proof}

\section{Hardness of Training (Non-negative) Sum of $k$ ReLUs}
\label{sec:hardness-coloring}

In this section, we consider the bounded sum of $k$-ReLU Training problem. Recall that this is the restriction of the bounded $k$ ReLU Training problem, in which we only allow the coefficient vector $\ba$ to be the all-one vector $\bone_k$; hence, here we are simply looking for a sum of $k$ ReLUs, i.e. $\sum_{j \in [k]} [\bw^j \cdot \bx]_+$. We prove the NP-hardness of the bounded sum of $k$-ReLU Training, as stated more precisely below. We note here that a hardness of similar form was already obtained in~\cite{vu1998infeasibility}, except that there each ReLU is allowed to have a bias term. Hence, for completeness, we include the full proof of the following theorem later in this section.

\begin{theorem} \label{thm:hardness-realizable}
For any constant $k \geq 2$, the bounded sum of $k$-ReLU Training problem is NP-hard.
\end{theorem}

Furthermore, we show that a tight running time lower bound for the task of bounded $k$-ReLU Training to within an error of $\varepsilon$ requires $2^{\Omega(1/\varepsilon)} \poly(n, m)$ time, even in the realizable case. 
Our training algorithm in the realizable case (Theorem~\ref{thm:realizable-training-short}) can be adapted to only consider $\ba = \bone_k$, with the same running time. Hence, our running time lower bound here is essentially tight in terms of $\eps$.

Our running time lower bound is based on the Gap Exponential Time Hypothesis (Gap-ETH)~\cite{Dinur16,ManurangsiR17}, which states that there is no $2^{o(n)}$-time algorithm that can distinguish between a satisfiable 3CNF formula and one which is not even $(1 - \delta)$-satisfiable for some constant $\delta > 0$. (We remark that the lower bound can also be based on the weaker Exponential Time Hypothesis (ETH)~\cite{IP01,IPZ01}, but the lower bound will only be of the form $2^{\Omega\left(\frac{1}{\varepsilon \cdot \poly\log(1/\varepsilon)}\right)} \poly(n, m)$.)

\begin{theorem} \label{thm:time-realizable}
Assuming Gap-ETH, for any constant $k \geq 2$, there is no $2^{o(1/\varepsilon)} \poly(n, m)$-time algorithm that can solve the bounded sum of $k$-ReLU Training problem within an additive square error of $\varepsilon$, even in the realizable case.
\end{theorem}

As the reader might have noticed, Theorems~\ref{thm:hardness-realizable} and~\ref{thm:time-realizable} are similar to Theorems~\ref{thm:realizable-hardness-short} and~\ref{thm:realizable-time-lower-bound-short}, except that the latter are for the bounded $k$-ReLU Training (where $\ba$ is not restricted to $\bone_k$). Indeed, we will use Theorems~\ref{thm:hardness-realizable} and~\ref{thm:time-realizable} to prove Theorems~\ref{thm:realizable-hardness-short} and~\ref{thm:realizable-time-lower-bound-short} in the upcoming section. 

Both Theorems~\ref{thm:hardness-realizable} and~\ref{thm:time-realizable} are based on a single reduction from the hypergraph $k$-coloring problem. Recall that, in the hypergraph $k$-coloring problem, we are given a hypergraph $G = (V, E)$ and the goal is to find a proper coloring $\chi: V \to [k]$. (A coloring $\chi$ is said to be proper if it does not result in any hyperedge $e$ being monochromatic, i.e., $|\chi(e)| = 1$.) The main properties of the reduction is given in the lemma below. As mentioned earlier in Section~\ref{sec:overview}, this reduction is in fact almost the same as that of~\cite{vu1998infeasibility}, except that the number of copies of each sample are different; this is needed in order to prove the tight running time lower bound (Theorem~\ref{thm:time-realizable}).

\begin{lemma} \label{lem:coloring-reduction}
For any integer $k \geq 2$, there exists a polynomial time reduction that takes in an $N$-vertex hypergraph $G$ whose edge size is at most $t$, and produces a multiset of samples $S = \{(\bx_i, y_i)\}_{i \in [m]} \subseteq \cB^n \times [0, 1]$ and a positive integer $\varepsilon \in \mathbb{R}^+$ such that
\begin{itemize}
\item (Completeness) If $G$ is $k$-colorable, then there exists $\bw^1, \dots, \bw^k \in \cB^n$ such that the samples are realizable by the sum of $k$ ReLUs $\sum_{j \in [k]} [\bw^j \cdot \bx]_+$ (i.e. $\cL(\bw^1, \dots, \bw^k; S) = 0$).
\item (Soundness) If every $k$-coloring of $G$ results in $\gamma$ fraction of edges being monochromatic for some $\gamma \in (0, 1)$, then $\cL(\bw^1, \dots, \bw^k; S) > \frac{\gamma}{100 k^2t^5} \cdot \frac{1}{N}$ for all $\bw^1, \dots, \bw^k \in \R^n$.
\end{itemize}
\end{lemma}

\begin{proof}
Let $V$ and $E$ denote the set of vertices and the set of hyperedges of $G$ respectively. Recall that we use $N$ and $M$ to denote $|V|$ and $|E|$ respectively. For convenience, let us rename the vertices as $1, 2, \dots, n$.

Let $n = N$ and\footnote{We use $\deg_G(i)$ to denote the number of hyperedges that $i$ belongs to.} $m = \sum_{i \in V} \deg_G(i) + |V|$. For each vertex $i \in V$, we create $\deg_G(i)$ copies of the sample $\bx_i = \be_i$ and label them by $y_i = \frac{1}{t \sqrt{t n}}$; we refer to such samples as the \emph{vertex $i$ samples}. Moreover, for each hyperedge $e = \{i_1, \dots, i_q\}$, we create a sample $(\bx_e, y_e)$ with $\bx_e = \frac{1}{\sqrt{t}} \sum_{j=1}^{q} \be_{i_j}$ and label it with $y_e = 0$; similarly, we refer to this as the \emph{hyperedge $e$ sample}. This completes our construction.

(Completeness) Suppose that the graph is $k$-colorable; let $\chi: V \to [k]$ be its proper $k$-coloring. We define $\bw^1, \cdots, \bw^k$ by $\bw^a_i = \frac{1}{t \sqrt{n}}$ iff $\chi(i) = a$ and $-\frac{1}{\sqrt{n}}$ otherwise. Consider the sum of $k$ ReLUs $\bx \mapsto [\bw^1 \cdot \bx]_+ + \cdots + [\bw^k \cdot \bx]_+$. For each vertex $i$, we have
\begin{align*}
[\bw^1 \cdot \bx_i]_+ \cdots + [\bw^k \cdot \bx_i]_+ = [\bw^1_i]_+ \cdots + [\bw^k_i]_+ = \frac{1}{t \sqrt{t n}}.
\end{align*}

Moreover, for each hyperedge $e = \{i_1, \cdots, i_q\}$, we have
\begin{align*}
[\bw^1 \cdot \bx_e]_+ + \cdots + [\bw^k \cdot \bx_e]_+ = \left[\frac{1}{\sqrt{t}} \sum_{j=1}^q \bw^1_{i_j}\right]_+ + \cdots + \left[\frac{1}{\sqrt{t}} \sum_{j=1}^q \bw^k_{i_j}\right]_+ = 0,
\end{align*}
where the second equality follows from the fact that the edge $e$ is not monochromatic.
Hence, the samples are realizable by a sum of $k$ ReLUs as desired.

(Soundness) Suppose contrapositively that for some constant $k \geq 2$, there exists a sum of $k$ ReLUs $\bx \mapsto \sum_{\ell \in [k]} [\bw^\ell \cdot \bx]_+$ such that $\cL(\bw^1, \dots, \bw^k; S) \leq \varepsilon := \frac{\gamma}{100 k^2t^5 n}$.

Let $T$ denote the set of vertices $i \in V$ such that there exists $\ell_i \in [k]$ where $\bw^{\ell_i}_i > \frac{1}{2k} \cdot \frac{1}{t\sqrt{tn}}$. Note that, for each $i \notin T$, we have $\sum_{\ell \in [k]} [\bw^\ell \cdot \bx_i]_+ = \sum_{\ell \in [k]} [\bw^\ell_i]_+ \leq \frac{1}{2t\sqrt{tn}}$. In other words, we incur a square loss of at least $\frac{1}{4t^3 n}$ for each copy of the vertex $i$ samples. This means that $\cL(\bw^1, \dots, \bw^k; S) \geq \frac{\sum_{i \in (V \setminus T)} \deg_G(i)}{m} \cdot \frac{1}{4t^3 n}$. From our assumption, this can be at most $\varepsilon$, which gives
\begin{align} \label{eq:uncolored-bound}
\sum_{i \in (V \setminus T)} \deg_G(i) \leq \varepsilon \cdot 4t^3 \cdot m \cdot n \leq \frac{\gamma M}{2},
\end{align}
where the latter comes from our choice of $\varepsilon$, and from $m \leq (t + 1)M$.

Now, consider the coloring $\chi: V \to [k]$ where we assign $\chi(i) = \ell_i$ for all $i \in T$, and assign $\chi(i)$ arbitrarily for all $i \notin T$. From~\eqref{eq:uncolored-bound}, the number of hyperedges that contain at least one vertex outside of $T$ is at most $\gamma M / 2$. Next, consider each hyperedge $e = \{i_1, \dots, i_q\}$ that is contained in $T$ (i.e. $e \subseteq T$). $e$ is monochromatic if and only if $\ell(i_1) = \cdots = \ell(i_q)$, which means that
\begin{align*}
\sum_{\ell \in [k]} [\bw^\ell \cdot \bx_e]_+ \geq [\bw^{\ell(i_1)} \cdot \bx_e]_+ = \left[\frac{1}{\sqrt{t}}\left(\bw^{\ell(i_1)}_{i_1} + \bw^{\ell(i_2)}_{i_2} + \cdots + \bw^{\ell(i_q)}_{i_q}\right)\right]_+ > \frac{1}{2kt^2 \sqrt{n}},
\end{align*}
where the last inequality comes from $\bw^{\ell(i)}_i > \frac{1}{2kt\sqrt{t n}}$ for all $i \in S$. In other words, each monochromatic hyperedge $e$ contained in $T$ incurs a square loss of more than $\frac{1}{4k^2t^4 n}$ in the hyperedge $e$ sample. As a result, the number of such hyperedges is less than
\begin{align*}
\varepsilon \cdot (4k^2 t^4 n) \cdot m \leq \frac{\gamma M}{2}.
\end{align*}
As a result, in total, the number of monochromatic hyperedges for the coloring $\chi$ is less than $\gamma M$. This concludes our proof.
\end{proof}


With the above reduction ready, Theorem~\ref{thm:hardness-realizable} and Theorem~\ref{thm:time-realizable} follow easily from known results on hardness of coloring. For Theorem~\ref{thm:hardness-realizable}, we may use the (classic) NP-hardness of coloring:

\begin{theorem}[\cite{lovasz-coloring,Stockmeyer-coloring}] \label{thm:hardness-coloring}
For any $k \geq 2$, deciding whether a given hypergraph $G$ is $k$-colorable is NP-hard. Furthermore, this holds even when $G$ has maximum edge size at most 3.
\end{theorem}

\begin{proof}[Proof of Theorem~\ref{thm:hardness-realizable}]
We reduce from hardness of coloring in Theorem~\ref{thm:hardness-coloring}. Let $G = (V, E)$ be the input hypergraph whose hyperedges are of size at most 3. By applying the reduction from Lemma~\ref{lem:coloring-reduction}, we get a set of samples $\{(\bx_i, y_i)\}_{i \in [m]}$. Lemma~\ref{lem:coloring-reduction} guarantees that, if the graph is $k$-colorable, then there is a sum of $k$ ReLUs $\sum_{j \in [k]} [\bw^j \cdot \bx]_+$ with $\bw^1, \dots, \bw^k \in \cB^n$ that realizes these samples. On the other hand, if $G$ is not $k$-colorable, then the soundness guarantee of Lemma~\ref{lem:coloring-reduction} implies that any sum of $k$ ReLUs incurs an average square error of at least $\frac{1/M}{100 \cdot k^2 \cdot 3^5} \cdot \frac{1}{N} = \frac{1}{24300 k^2 NM}$. Hence, the bounded sum of $k$ ReLU Training problem is also NP-hard.
\end{proof}

\begin{remark}
If we plug in hardness of approximate coloring (e.g.~\cite{DinurRS05,bhangale2018np}) to our reduction instead of the hardness of exact coloring in Theorem~\ref{thm:hardness-coloring}, then we can actually get a stronger soundness where the constructed samples are not realizable even for any sum of $k'$ ReLUs for any constant $k' > k$. In fact, using the hardness of approximation of coloring in~\cite{bhangale2018np}, $k'$ can even be taken as large as $(\log n)^{1 - o(1)}$.

We note, however, that such strong soundness does not hold for the problem of bounded $k$-ReLU Training (with possibly negative coefficient). In particular, our gadget in Lemma~\ref{lem:gadget-neg} has a soundness guarantee that only holds against $k$-ReLU, but not even $(k + 1)$-ReLU. It remains an interesting open question to extend such stronger soundness to this case as well.
\end{remark}

We now move on to prove our running time lower bound (Theorem~\ref{thm:time-realizable}). To prove this result, we will use the following running time lower bound, which is explicit in~\cite{Patrank94}:

\begin{theorem}[\cite{Patrank94}] \label{thm:eth-coloring}
Assuming Gap-ETH, for any $k \geq 2$, there exists $\gamma > 0$ such that the following holds. There is no $2^{o(N)}$ time algorithm that can, given an $N$-vertex $(k + 1)$-uniform hypergraph $G$, distinguish between the following two cases:
\begin{itemize}
\item (Completeness) $G$ is $k$-colorable.
\item (Soudness) Any $k$-coloring of $G$ violates more than $\gamma$ fraction of its hyperedges.
\end{itemize}
\end{theorem}

We remark that, strictly speaking, Petrank only proved the above theorem in the case of $k = 2$, for which the problem of 2-coloring 3-uniform hypergraph is equivalent to the so-called Max NAE 3SAT, which was proved to be hard to approximate in~\cite[Theorem 4.3]{Patrank94}; the reduction is a linear time reduction from the gap version of 3-SAT, which yields the running time lower bound we stated above. Nonetheless, it is also very simple to generalize the result to the case $k \geq 2$. We sketch the argument in Appendix~\ref{app:coloring}.

\begin{proof}[Proof of Theorem~\ref{thm:time-realizable}]
We reduce from hardness of coloring in Theorem~\ref{thm:hardness-coloring}. Let $G = (V, E)$ be a $(k + 1)$-uniform hypergraph. By applying the reduction from Lemma~\ref{lem:coloring-reduction}, we get a set of examples $\{(\bx_i, y_i)\}_{i \in [m]}$. Lemma~\ref{lem:coloring-reduction} guarantees that, if the graph is $k$-colorable, then there is a sum of $k$ ReLUs $\sum_{j \in [k]} [\bw^j \cdot \bx]_+$ where $\bw^1, \dots, \bw^k \in \cB^n$ that realizes these samples. On the other hand, if any $k$-colorable of $G$ results in at least $\gamma$ fraction of the hyperedges being monochromatic, then the soundness guarantee of Lemma~\ref{lem:coloring-reduction} implies that any sum of $k$ ReLUs incurs an average square error of more than $\varepsilon := \frac{\gamma}{100 k^2 t^5} \cdot \frac{1}{N} = \Omega_{\gamma, k}\left(\frac{1}{N}\right)$.

Hence, if there is a $2^{o(1/\varepsilon)} \poly(n)$-time learning algorithm for 2-ReLUs (in the realizable case) to within square error of $\varepsilon$, then we can distinguish the two cases in Theorem~\ref{thm:eth-coloring} in time $2^{o(N)}$ time. By Theorem~\ref{thm:eth-coloring}, this violates Gap-ETH.
\end{proof}

\subsection{On hardness of coloring}
\label{app:coloring}

In this section, we briefly sketch the proof of Theorem~\ref{thm:hardness-coloring}. First, Theorem 4.3 of~\cite{Patrank94} immediately implies the following.

\begin{theorem}[\cite{Patrank94}] \label{thm:eth-coloring-2-color}
Assuming Gap-ETH, there exists $\gamma > 0$ such that the following holds. There is no $2^{o(N)}$ time algorithm that can, given an $N$-vertex 3-uniform hypergraph $G$, distinguish between the following two cases:
\begin{itemize}
\item (Completeness) $G$ is 2-colorable.
\item (Soudness) Any 2-coloring of $G$ violates more than $\gamma$ fraction of its hyperedges.
\end{itemize}
\end{theorem}

We may now prove Theorem~\ref{thm:hardness-coloring} as follows.

\begin{proof}[Proof Sketch of Theorem~\ref{thm:hardness-realizable}]
Given an $N$-vertex 3-uniform input hypergraph $G = (V, E)$ from Theorem~\ref{thm:eth-coloring-2-color}. We create a new hypergraph $G' = (V', E')$ where $V'$ is simply $V$ together with $k - 2$ additional ``dummy vertices'' (i.e. $V' = V \cup \{u_1, \dots, u_{k - 2}\}$). We then let $E' = \{e \cup \{u_1, \dots, u_{k - 2}\} \mid e \in E\}$. It is simple to check that, for any $\nu \in [0, 1]$, there exists a $k$-coloring of $G'$ with $\nu$ fraction of edges being monochromatic if and only if there exists a 2-coloring of $G$ with $\nu$ fraction of edges being monochromatic.
\end{proof}

\section{Handling Negative Coefficients: Hardness of $k$-ReLU Training}
\label{sec:neg}

In this section, we show that our hardness from the previous section can be easily extended to the case where the coefficients in front of each ReLU unit is allowed to be negative. Specifically, we will prove Theorems~\ref{thm:realizable-hardness-short} and~\ref{thm:realizable-time-lower-bound-short} here. 


The main gadget used to translate our results from the non-negative coefficient case to the more generalized case here is just a set of points that can be realized by a weighted sum of $k$ ReLUs only when all the coefficients are positive, as stated below.

\begin{lemma}[Main Gadget] \label{lem:gadget-neg}
For any $k \in \N$, there exists a set of samples $\tS = \{(\tbx_i, \ty_i)\}_{i \in [\tm]} \subseteq \cB^k \times [0, k]$ and a positive real number $\tau \in \R^+$ such that
\begin{itemize}
\item (Completeness) The samples can be realized by $\sum_{j \in [c]} [\tbw^j \cdot \tbx]_+$ for some $\tbw^1, \dots, \tbw^k \in \cB^k$.
\item (Soundness) For any $\tbw^1, \dots, \tbw^k \in \R^k$, $\ba \in \{-1, 1\}^n \setminus \{\bone_k\}$, $\cL(\tbw^1, \dots, \tbw^k, \ba; \tS) \geq \tau$.
\end{itemize}
Moreover, the set $\tS$ can be constructed in time $2^{O(k)}$.
\end{lemma}

We will construct our gadget in the above lemma in Section~\ref{sec:gadget-negative-weight}. Before we do so, let us use it to prove Theorem~\ref{thm:realizable-hardness-short}. The main idea is simple: we start from the NP-hard instance from the non-negative weights case and extend the dimension by $k$ (where we simply add $k$ zeros to the end of each sample). Then, we construct additional samples using the gadget (Lemma~\ref{lem:gadget-neg}); the gadget samples are embedded in the last $k$ coordinates and the remaining coordinates are just zeros. The key observation here is that, if a weighted sum of $k$ ReLUs with negative weights is used, then it must incur the error from the soundness of Lemma~\ref{lem:gadget-neg}. Otherwise, we are back to the case where all coefficients are non-negative, for which we already know the hardness. 

At this point, we would also like to remark that, while in all our proofs we only consider constant $k \geq 2$, we can in fact take $k$ to be as large as $O(\log n)$ for Theorem~\ref{thm:realizable-hardness-short}. The bottleneck here is the construction time $2^{O(k)}$ of the gadget in Lemma~\ref{lem:gadget-neg}; since we want this to be polynomial time, we can take $k$ to be at most $O(\log n)$.

\begin{proof}[Proof of Theorem~\ref{thm:realizable-hardness-short}]
Let $S = \{(\bx_i, y_i)\}_{i \in [m]}$ where $\bx_i \in \{0, 1\}^n$ be the NP-hard instance from Theorem~\ref{thm:hardness-realizable}, and let $\tS = \{(\tbx_i, \ty_i)\}_{i \in [\tm]}$ where $\bx_i \in \cB^k$ be the samples from Lemma~\ref{lem:gadget-neg}. We construct the multiset of new samples $\hS = \{(\hbx_i, \hy_i)\}_{i \in [\hm]}$ where $\hbx_i \in \cB^{\hn}$ as follows.
\begin{itemize}
\item Let $\hn = n + k$ and $\hm = m + \tm$.
\item For every $i \in [m]$, we add $\tm$ copies of the labelled sample $(\bx_i \circ \bzero_{k}, 0.5y_i)$ in $\hS$.
\item For every $i \in [\tm]$, we add $m$ copies of the labelled sample $(\bzero_{n} \circ \tbx_i, 0.5\ty_i)$ in $\hS$.
\end{itemize}

\paragraph{(Completeness)}
Suppose that $\{(\bx_i, y_i)\}_{i \in [m]}$ are realizable by a sum of $k$-ReLUs $\sum_{j \in [k]} [\bw^j \cdot \bx]_+$ where $\bw^1, \dots, \bw^k \in \cB^n$. Recall also from Lemma~\ref{lem:gadget-neg} that the samples $\{(\tbx_i, y_i)\}_{i \in [\tm]}$ can be realized by a sum of $k$-ReLUs $\sum_{j \in [k]} [\tbw^j \cdot \tbx]_+$ where $\tbw^1, \dots, \tbw^k \in \cB^k$. For every $j \in [k]$, let $\hbw^j = 0.5\bw^j \circ 0.5\tbw^j$. It is easy to check that $\|\hbw^1\|, \dots, \|\hbw^k\| \leq 1$ and that the constructed samples can be realized by the $k$-ReLU $\sum_{j \in [k]} [\hbw^j \cdot \hbx]_+$.

\paragraph{(Soundness)}
Suppose that any sum of $k$-ReLUs incurs an average loss of at least $\frac{1}{\poly(n)}$ on the samples $\{(\bx_i, y_i)\}_{i \in [m]}$. We will show that any weighted sum of $k$-ReLUs $\sum_{j \in [k]} a_j [\hbw_j \cdot \hbx]_+$ incurs a loss of at least $\frac{1}{\poly(\tn)}$ on the constructed samples $\{(\hbx_i, \hy_i)\}_{i \in [\hm]}$. Consider the following two cases:
\begin{itemize}
\item $a_1 = \cdots = a_k = +1$. In this case, the assumption on $\{(\bx_i, y_i)\}_{i \in [m]}$ implies that the weighted sum of $k$-ReLUs incurs an average loss of at least $\frac{1}{\poly(n)}$ on the samples $\{(\bx_i \circ \bzero_{\tm}, 0.5y_i)\}_{i \in [m]}$. Since (the copies of) these samples contribute to half of the total number of constructed samples, we have that the weighted sum of $k$-ReLUs incurs an average loss of at least $\frac{1}{2} \cdot \frac{1}{\poly(n)} \geq \frac{1}{\poly(\tn)}$ on the constructed samples $\{(\hbx_i, \hy_i)\}_{i \in [\hm]}$.
\item $a_j = -1$ for some $j \in [k]$. In this case, the soundness of Lemma~\ref{lem:gadget-neg} implies that the weighted sum of $k$-ReLUs incurs an average loss of at least $\tau$ on the samples $\{(\bzero_m \circ \tbx_i, 0.5\ty_i)\}_{i \in [\tm]}$. Since (the copies of) these samples contribute to half of the total number of constructed samples, we have that the weighted sum of $k$-ReLUs incurs an average loss of at least $0.5\tau$ on the constructed samples $\{(\hbx_i, \hy_i)\}_{i \in [\hm]}$. Finally, recall that $\tau$ is a positive constant that depends only on $c$; hence, $0.5\tau \geq 1/\poly(\tn)$. 
\end{itemize}
\end{proof}

We remark that it is straightforward to see that, when plugging the above reduction into Theorem~\ref{thm:time-realizable}, we also immediately get Theorem~\ref{thm:realizable-time-lower-bound-short}. We omit the full proof here, but note that the main observation is that, the error incurred in the second case ($a_j = -1$ for some $j \in [k]$) is an absolute constant that only depends on $k$. This means that, for any sufficiently small $\epsilon$, we are forced to use $a_1 = \cdots = a_k = 1$, which takes us right back to Theorem~\ref{thm:time-realizable}.

\subsection{Construction of the Gadget}
\label{sec:gadget-negative-weight}

We now move on to the construction of our main gadget (Lemma~\ref{lem:gadget-neg}). Before we state the proof for general $k$, let us first note that the case for $k = 2$ is incredibly simple: just create two samples $(\bu, 1)$ and $(-\bu, 1)$ where $\bu$ can be any unit vector. (This construction is in fact the same as a gadget used in~\cite{Woodruf}.) These samples can be realized by the sum of 2 ReLUs $[\tbw^1 \cdot \tbx]_+ + [\tbw^2 \cdot \tbx]_+$ where $\tbw_1 = \bu$ and $\tbw_2 = -\bu$. To see that they cannot be realized by weighted sum of 2 ReLUs with a negative coefficient $[\tbw_1 \cdot \tbx]_+ - [\tbw_2 \cdot \tbx]_+$, observe that $\tbw_1 \cdot \bu$ or $\tbw_1 \cdot (- \bu)$ must be non-positive, meaning that it must output a non-positive value on at least one of $\bu$ or $-\bu$. Hence, the samples are not realizable by such a weighted sum of 2 ReLUs.

The above example is a special case of a more general phenomenon: if we look at the ``sign pattern'' (i.e. whether $\tbw^1 \cdot \tbx, \dots, \tbw^k \cdot \tbx$ are positives), there can be as many as $2^k$ such patterns; $2^k - 1$ such patterns may result in a positive output in the (positive) sum of $c$ ReLUs, with the only exception being when $\tbw_1 \cdot \tbx, \dots, \tbw_c \cdot \tbx < 0$. However, if we look at the sign patterns for a weighted sum of $k$ ReLUs with at least one negative coefficient, then only at most $2^k - 2$ patterns can result in positive outputs. (For instance, if the coefficient of $[\tbw^k \cdot \tbx]_+$ is negative, then the sign pattern $\tbw^1 \cdot \tbx, \dots, \tbw^{k - 1} \cdot \tbx < 0$ and $\tbw_k \cdot \tbx \geq 0$ cannot result in a positive output.) Although we do not use these bounds directly in the above samples for $k = 2$, we do use the fact that $\bu$ and $-\bu$ cannot correspond to the same sign pattern. Roughly speaking, our soundness proof for the general case below also proceeds by arguing that, since there are fewer sign patterns that result in positive outputs when there is a negative coefficient, pigeonhole principle implies that some samples that should not be from the same sign pattern must be from the same sign pattern when there is a negative coefficient, which would lead to a large square error similar to the case $k = 2$ above.

To formalize our construction, recall that a set of vectors in $\R^d$ is said to be in \emph{general position} if any $d$ of these vectors are linearly independent. We use $\cB(\bx, r) := \{\bx' \mid \|\bx - \bx'\| \leq r\}$ to denote the ball of radius $r$ around $\bx$. It is well known that for any $d, t \in \N$, $\bx \in \R^d$ and $r \in \R^+$, it is possible to construct a set of $t$ vectors in $\cB(\bx, r)$ that are in general position in $poly(d, t)$ time.

\begin{proof}[Proof of Lemma~\ref{lem:gadget-neg}]
Let $f(\tbx) = [\tx_1]_+ + \cdots + [\tx_k]_+$ be the sum of $k$-ReLUs, in which the $j$-th ReLU weight vector $\tbw_j$ is just the $j$-th standard basis vector. We construct our samples as follows:
\begin{itemize}
\item For every $\bu \in \{-1, +1\}^k \setminus \{-\bone_k\}$, let $S_{\bu} \subseteq \cB(\frac{\bu}{2\sqrt{k}}, \frac{0.01}{k})$ be the set of any $2^k \cdot k$ points in general position. (As stated before the proof, this can be constructed in $2^{O(k)}$ time.)
\item Let $S := \bigcup_{\bu \in \{-1, +1\}^k \setminus \{-\bone_k\}} S_{\bu}$. The samples are $(\tbx, f(\tbx))$ for all $\tbx \in S$.
\end{itemize}

Before we prove the completeness and soundness of the gadget, we first define the real number $\tau$ that will be used in the soundness.  We let
\begin{align*}
\tau = \frac{0.1}{k|S|} \cdot \min\left\{1, \min_{\bu \in \{-1, +1\}^n \setminus \{-\bone_k\} \atop S'_{\bu} \subseteq S_{\bu}, |S'_{\bu}| = k} \inf_{\|\bw\| = 1} \sum_{\bx \in S'_{\bu}} \|\bw \cdot \bx\|^2 \right\}.
\end{align*}
A priori, it might not be clear that $\tau$ has to be positive. To see that this is the case, observe that $\inf_{\|\bw\| = 1} \sum_{\bx \in S'_{\bu}} \|\bw \cdot \bx\|^2$ is exactly equal to $\min_{j=1, \dots, k} \lambda_j^2$ where $\lambda_j$ is the $j$-th eigenvalue of the $(k \times k)$-matrix whose rows are $\bx \in S_{\bu}'$. Since the vectors in $S_{\bu}$ are in general position, this matrix must be full rank, which implies that all its eigenvalues are non-zero. As a result, we have $\inf_{\|\bw\| = 1} \sum_{\bx \in S'_{\bu}} \|\bw \cdot \bx\|^2 > 0$, which in turn implies that $\tau > 0$ as desired. (Note here that $\tau$ only depends on $k$ and our choices of $\{S_{\bu}\}_{\bu}$.)

\paragraph{(Completeness)} By construction, the samples are realizable by $f$, which is a sum of $k$ ReLUs.

\paragraph{(Soundness)} We now consider any weighted sum of $k$ ReLUs, such that at least one of the coefficient is negative. Without loss of generality, we may assume that this is a function of the form $g(\tbx) = \sum_{j=1}^{k'} [\tbw^j \cdot \tbx]_+ - \sum_{j=k'+1}^{k} [\tbw^j \cdot \tbx]_+$ for some non-negative integer $k' < k$. We will show that $g$ incurs an average loss of at least some positive constant. To do so, we define the following notations: let $\sgn(z) = +1$ if $z > 0$ and 0 otherwise. For every ``sign pattern'' $\bs \in \{0, 1\}^k$, let $P_{\bs} \subseteq \R^k$ denote the subsets of points $\tbx \in \R^k$ such that $\sgn(\tbw^j \cdot \tbx) = s_j$ for all $j \in [k]$. Now, consider the following two cases:
\begin{itemize}
\item Case I: $(P_{(0, 0, \dots, 0, 0)} \cup P_{(0, 0, \dots, 0, 1)}) \cap S \ne \emptyset$. Let $\tbx$ be any element of $(P_{(0, 0, \dots, 0, 0)} \cup P_{(0, 0, \dots, 0, 1)}) \cap S$. From $\tbx \in S$, it is simple to check that\footnote{In particular, if $\tbx \in S_{\bu}$ and $u_i = 1$, then we must have $\tx_i \geq 0.4/\sqrt{k}$, which implies that $f(\tbx) \geq 0.4/\sqrt{k}$.} $f(\tbx) \geq \frac{0.4}{\sqrt{k}}$. However, from $\tbx \in P_{(0, 0, \dots, 0, 0)} \cup P_{(0, 0, \dots, 0, 1)}$, we have $g(\tbx) \leq 0$. Hence, $g$ must incur an average square loss of at least $\frac{1}{|S|} \cdot \frac{0.16}{k}$, which is at least $\tau$ by the definition of the latter.
\item Case II: $(P_{(0, 0, \dots, 0, 0)} \cup P_{(0, 0, \dots, 0, 1)}) \cap S = \emptyset$.

Fix any $\bu \in \{-1, +1\}^k \setminus \{-\bone_k\}$, since $|S_{\bu}| = 2^k \cdot k$ and $\cup_{\bs \in \{0, 1\}^k} P_{\bs} = \R^k$, there must exists a sign pattern $\bs_{\bu}$ such that $|S_{\bu} \cap P_{\bs_{\bu}}| \geq k$. From the assumption of this case, we must also have that $\bs_{\bu} \ne (0, 0, \dots, 0), (0, 0, \dots, 0, 1)$.

As a result, by pigeonhole principle\footnote{Note that there are $2^k - 1$ pigeons and only $2^k - 2$ holes.}, there exists two distinct $\bu^1, \bu^2 \in \{-1, +1\}^k \setminus \{-\bone_k\}$ such that $\bs_{\bu^1} = \bs_{\bu^2}$. Let $\bs^* = \bs_{\bu^1} = \bs_{\bu^2}$; we have that $|P_{\bs^*} \cap S_{\bu^1}|, |P_{\bs^*} \cap S_{\bu^2}| \geq k$.

Furthermore, observe that, for all $\tbx \in P_{\bs^*}$, we have
\begin{align*}
g(\tbx) &= \sum_{j=1}^{k'} s^*_j \cdot \tbw^j \cdot \tbx - \sum_{j=k' + 1}^{k} s^*_j \cdot \tbw^j \cdot \tbx \\
&= \left(\sum_{j=1}^{k'} s^*_j \cdot \tbw^j - \sum_{j=k' + 1}^{k} s^*_j \cdot \tbw^j\right) \cdot \tbx \\
&= \tbw^* \cdot \tbx,
\end{align*}
where we define $\tbw^* = \left(\sum_{j=1}^{k'} s^*_j \cdot \tbw^j - \sum_{j=k' + 1}^{k} s^*_j \cdot \bw^j\right)$. As a result, the average square error incurred by $g$ on the constructed samples is at least
\begin{align} \label{eq:g-err-lb}
\frac{1}{|S|} \left(\sum_{\bx \in P_{\bs^*} \cap (S_{\bu^1} \cup S_{\bu^2})} (f(\tbx) - g(\tbx))^2\right)
= \frac{1}{|S|} \sum_{\bx \in P_{\bs^*} \cap (S_{\bu^1} \cup S_{\bu^2})} (f(\tbx) - \bw^* \cdot \tbx))^2.
\end{align}
Next, observe that, for all $\tbx \in S_{\bu^1}$, we have $f(\tbx) = \sgn(\bu^1) \cdot \bx$ where define $\sgn(\bu)$ as $(\sgn(u_j))_{j \in [k]}$. Similarly, for all $\tbx \in S_{\bu^2}$, we have $f(\tbx) = \sgn(\bu^2) \cdot \tbx$. Moreover, since $\bu^1 \ne \bu^2$ both belong to $\{\pm 1\}^k$, we must have $\|\sgn(\bu^1) - \sgn(\bu^2)\|_2 \geq 2$. From this, we must have either $\|\bw^* - \sgn(\bu^1)\| \geq 1$ or $\|\bw^* - \sgn(\bu^2)\| \geq 1$; without loss of generality, we assume the former. We may lower bound the right hand side term in~\eqref{eq:g-err-lb} by
\begin{align*}
\frac{1}{|S|} \sum_{\bx \in P_{\bs^*} \cap S_{\bu^1}} \|(\sgn(\bu^1) - \tbw^*) \cdot \tbx\|^2
&\geq \frac{1}{|S|} \inf_{\|\bw\| = 1} \sum_{\tbx \in P_{\bs^*} \cap S_{\bu^1}} \|\bw \cdot \tbx\|^2 \\
&\geq \tau,
\end{align*}
where the first inequality follows from $\|\sgn(\bu^1) - \bw^*\| \geq 1$ and the last inequality follows from $|P_{\bs^*} \cap S_{\bu^1}| \geq k$ and from our definition of $\tau$ (especially the second term with $\bu = \bu^1$).
\end{itemize}
Hence, in both cases, we have that the average squared error incurred by $g$ must be at least $\tau$, which concludes our proof.
\end{proof}

\section{Training and learning algorithms} \label{sec:learning}

In this Section, we describe algorithms for learning and training ReLUs. We begin in Section~\ref{subsec:simple-training-algos} by giving a simple training algorithm, whose running time is ostensibly super-polynomial. Then, in Section~\ref{subsec:simple-training-algos}, we define the learning problem and show, using standard generalization arguments, that in fact the aforementioned training algorithm gives the claimed running time lower bound (in Theorems~\ref{thm:single-time-lower-bound-short} and~\ref{thm:realizable-time-lower-bound-short}).

\subsection{A Simple Training Algorithm}
\label{subsec:simple-training-algos}

In light of the NP-hardness from the previous sections, a polynomial time algorithm for training depth-2 ReLUs do not exist (unless P $=$ NP). Nevertheless, it is still possible to train the ReLUs in super polynomial time. For instance, ~\cite{arora2018understanding} gives a simple algorithm that runs in time $n^{O(km)}$ and output the optimal training error (to within arbitrarily small accuracy). Below, we observe that their approach also yields an $2^{km} \cdot poly(n, m, k)$ time algorithm.
Before we proceed to the statement and the proof of the algorithm, we remark that, the NP-hardness proof from Section~\ref{sec:hardness-coloring} in fact implies that, assuming the Exponential Time Hypothesis (ETH)~\cite{IP01,IPZ01}\footnote{ETH states that 3SAT with $n$ variables and $m = O(n)$ clauses cannot be solved in $2^{o(n)}$ time.}, the bounded $k$-ReLU training problem cannot be solved exactly in $2^{o(m)}$ time for any constant $k \geq 2$. Hence, the dependency $m$ in the exponent is tight in this sense. However, it is unclear whether the dependency on $k$ can be improved.

\begin{lemma} \label{lem:exp-algo}
There is an $2^{k(1 + m)} \cdot poly(n, m,1/\delta,C)$-time algorithm that, given samples $\{(\bx_i, y_i)\}_{i \in [m]}$ where $\bx_i \in \mathbb{R}^n$ and an accuracy parameter $\delta \in (0,1)$, finds the weights $\bw_1, \dots, \bw_k \in \cB^n, \ba \in \{-1, 1\}^k$ that minimizes $\cL(\bw^1, \dots, \bw^k, \ba; S)$ up to an additive error of $\delta$. We assume the bit complexity of every number appearing  in the coordinates of the $\bx_i$'s and $y_i$'s is at most $C$.
\end{lemma}
\begin{proof}
First, we iterate over all possible $\ba \in \{-1, 1\}^k$.
Moreover, for each ReLU term $[\left<\bw_j, \bx_i\right>]_+$ guess whether it equals $0$ or $\left<\bw_j, \bx_i\right> + b_j$ and replace the term in the error function accordingly.
Furthermore, if the guess $[\left<\bw_j, \bx_i\right>]_+=0$ was made then add the linear constraint $\left<\bw_j, \bx_i\right>\leq 0$. Else, add the linear constraint $\left<\bw_j, \bx_i\right>\geq 0$.
After all guesses are made we get a convex optimization program with linear constraints. It is well known that such a convex optimization problem can be solved in time polynomial in $n,m,1/\delta,C$ using a separation oracles and the ellipsoid algorithm (see for example, \cite[Section 2.1]{bubeck2015convex}). Since the number of guesses is at most $2^k \cdot (2^m)^k$, the claim follows.
\end{proof}

\subsection{Learning $k$-ReLUs}
\label{subsec:learning-relus}

We will now use the above training algorithms to give learning algorithms for ReLUs. We follow the agnostic learning model for real-valued function from~\cite{Haussler92,KearnsSS94}. 
A \emph{concept class} $\cC: \cY^{\cX}$ is any set of functions from $\cX$ to $\cY$. We say that a concept class $\cC$ is \emph{properly agnostically learnable} with respect to loss function $\ell: \cY \times \cY \to \mathbb{R}^+$ if, for every $\delta, \varepsilon > 0$, there is an algorithm $\cA$ such that, for any distribution $\cD$ over $\cX \times \cY$, takes in independent random samples from $\cD$ and outputs a hypothesis $h \in \cC$ such that, with probability $1 - \delta$, the following holds: $$\cL(h; \cD) \leq \min_{c \in \cC} \cL(c; \cD) + \varepsilon$$ where $\cL(f; \cD) := \E_{(\bx, y) \sim \cD}[\ell(f(\bx), y)]$ is the expected loss for $f$ over $\cD$. Furthermore, if $\cX \subseteq \mathbb{R}^n$ and the algorithm $\cA$ runs in time polynomial in $n$ and $1/\delta$, then $\cC$ is said to be \emph{efficiently} properly agnostically learnable. Throughout this section, we only consider the quadratic loss function (i.e., $\ell(y, y') = (y - y')^2$) and this will henceforth not be explicitly stated.
 
The concept classes we consider are the classes of sums of $k$ ReLUs, where each coefficient has magnitude at most one, and the distribution $\cD$ is allowed to be any distribution on the ball. More specifically, the class $k$-ReLU$(n)$, which represent the sums of $k$ ReLUs, is defined as follows:

\begin{definition}[$k$-ReLU($n$)]
For any $n, k \in \mathbb{N}$ and any $\bw^1, \dots, \bw^k \in \mathbb{R}^n, \ba \in \{-1, 1\}^k$, we use $\relu_{\bw^1, \dots, \bw^k, \ba}: \cB^n \to [-k, k]$ to denote the function $\relu_{\bw^1, \dots, \bw^k, \ba}(\bx) = \sum_{j=1}^k a_j[\left<\bw_j, \bx\right>]_+$.
Let $k$-ReLU$(n)$ denote the class $\{\relu_{\bw^1, \dots, \bw^k, \ba} \mid \bw^1, \dots, \bw^k \in \cB^n, \ba \in \{-1, 1\}^k\}$.
\end{definition}

We show that, for any fixed number of ReLUs $k$, the class above can be efficiently agnostically properly learned, as stated below.

\begin{theorem} \label{thm:learning}
For any $n, k \in \mathbb{N}$, ReLU$(n, k)$ can be efficiently agnostically properly learned for the quadratic loss function in time $2^{O(k^5/\varepsilon^2)} \cdot (n/\delta)^{O(1)}$ time.
\end{theorem}

When the sum of ReLU's is realizable, we can get better running time both in terms of $k, \epsilon$:

\begin{theorem} \label{thm:learning-realizable}
For any $n, k \in \mathbb{N}$, ReLU$(n, k)$ can be efficiently properly learned in the realizable case for the quadratic loss function in time $2^{O(k^3/\varepsilon \cdot \log^3(k/\eps))} \cdot (n/\delta)^{O(1)}$ time.
\end{theorem}

We remark that learning algorithms immediately imply training algorithms, by simplying letting $\cD$ be the uniform distribution on the input set of labelled samples $S = \{(\bx_i, y_i)\}_{i \in [m]}$. Thereby, Theorems~\ref{thm:learning} and~\ref{thm:learning-realizable} imply Theorems~\ref{thm:agnostic-training-short} and~\ref{thm:realizable-training-short}, respectively.

\subsubsection{Generalization Bounds}

Before we get to our proofs, we state the necessary generalization bounds.

\begin{theorem}[\cite{BM02}] \label{thm:gen-err}
Let $\cD$ be a distribution over $\cX \times \cY$ and let $\ell: \cY \times \cY \to \mathbb{R}$ be a $b$-bounded loss function that is $L$-Lispschitz in its first argument. Let $\cF \subseteq (\cY')^{\cX}$ and for any $f \in \cF$, let $\cL(f; \cD) := \E_{(\bx, y) \sim \cD}[\ell(f(\bx), y)]$ and $\cL(f; S) := \frac{1}{m} \sum_{i=1}^m \ell(f(\bx_i), y_i)$, where each sample $(\bx_i, y_i) \in S$ is drawn independently uniformly at random according to $\cD$. Then, for any $\delta > 0$, with probability at least $1 - \delta$, the following is true for all $f \in \cF$:
\begin{align} \label{eq:gen-err}
|\cL(f; \cD) - \cL(f; S)| \leq 4 \cdot L \cdot \cR_m(\cF) + 2 \cdot b \cdot \sqrt{\frac{\log(1/\delta)}{m}}
\end{align}
where $\cR_m(\cF)$ is the Rademacher complexity of the function class $\cF$.
\end{theorem}

While the above bound is generic and easy to apply it turns out to be not tight, especially when $\cL(f; S)$ is small. Below we list such a bound from~\cite{SrebroST10}; for simplicity of presentation, we only state the bound when $\cL(f; S) = 0$ which suffices for us. To state the bound, we also require the notion of smoothness of the loss; we say that a loss $\ell$ is \emph{$H$-smooth} if it is differentiable in the first variable and the derivative is $H$-Lipchitz.

\begin{theorem}[\cite{SrebroST10}] \label{thm:gen-err-realizable}
Let $\cD, S, \ell, \cF, \cL, \cR_m(\cF)$ be as in Theorem~\ref{thm:gen-err}. Furthermore, assume that $\ell$ is $H$-smooth. Then, for any $\delta > 0$, with probability at least $1 - \delta$, the following is true for all $f \in \cF$ such that $\cL(f; S) = 0$:
\begin{align} \label{eq:gen-err-realizable}
\cL(f; \cD) \leq C \left(H \log^3 m \cdot \cR_m(\cF)^2 + \frac{b \cdot \log(1/\delta)}{m}\right),
\end{align}
where $C > 1$ is an absolute constant.
\end{theorem}

To see the differences between Theorems~\ref{thm:gen-err} and~\ref{thm:gen-err-realizable}, notice that, if we ignore the second term in~\eqref{eq:gen-err-realizable} for the moment, we only require $m = O_{b, \delta}(1/\eps)$ to get $\cL(f; \cD) \leq \eps$ in the latter whereas the former would need $m = O_{b, \delta}(1/\eps^2)$. This will indeed result in the difference in the running time of the learning algorithms for $k$-ReLUs in the realizable versus agnostic case.

Finally, we also use the following bounds on the Rademacher complexity:

\begin{theorem}[\cite{KST08}] \label{thm:KST}
Let $\cX \subseteq \cB^n$ and let $\cW = \{\bx \mapsto \left<\bx, \bw\right> \mid \|w\|_2 \leq 1\}$. Then,
\begin{align*}
\cR_m(\cW) \leq \sqrt{1/m}.
\end{align*}
\end{theorem}

\begin{fact} \label{fact:subadd}
Let $\cF_1, \cF_2 \subseteq \mathbb{R}^{\cX}$ be any function classes and let $\cF = \{f_1 + f_2 \mid f_1 \in \cF_1, f_2 \in \cF_2\}$. Then,
\begin{align*}
\cR_m(\cF) \leq \cR_m(\cF_1) + \cR_m(\cF_2).
\end{align*}
\end{fact}

\begin{theorem}[\cite{BM02,LT91}] \label{thm:lips}
Let $\psi: \mathbb{R} \to \mathbb{R}$ be Lipschitz with constant $L_{\psi}$ and suppose that $\psi(0) = 0$. Let $\cY \subseteq \cR$, and for a function $f \in \cY^{\cX}$, let $\phi \circ f$ denote the composition of $\psi$ and $f$. For $\cF \subseteq \cY^{\cX}$, let $\psi \circ \cF = \{\psi \circ f \mid f \in \cF\}$. It holds that $\cR_m(\psi \circ \cF) \leq 2 \cdot L_\psi \cdot \cR_m(\cF)$.
\end{theorem}

\subsubsection{Properly Learning ReLUs: The Agnostic Case}
Our proof of Theorem~\ref{thm:learning} follows by an application of a standard generalization argument to bound the sample complexity given the algorithm in Lemma \ref{lem:exp-algo}. 
It turns out that in our case, the number of samples needed only depends on $k$ and $\varepsilon$ (and not the dimension $n$) due to boundedness of our networks. Hence, by applying the algorithm from Lemma~\ref{lem:exp-algo}, we immediately get Theorem~\ref{thm:learning}.

We now proceed to prove Theorem~\ref{thm:learning}. For ease of presentation, when we invoke the algorithm from Lemma~\ref{lem:exp-algo}, we will ignore the accuracy parameter $\delta$ and pretend that the algorithm output an actual optimal solution. The presence of $\delta$ only adds an additive term which we can make sufficiently small.

\begin{proof}[Proof of Theorem~\ref{thm:learning}]
First, let us describe the algorithm. Given samples $S = \{(\bx_i, y_i)\}_{i \in [m]}$ where
\begin{align*}
m = \left\lceil \frac{1024 \cdot k^4 \cdot (1 + \log(1/\delta))}{\varepsilon^2} \right\rceil.
\end{align*}
We use the algorithm from Lemma~\ref{lem:exp-algo} to solve for $\bw_1, \dots, \bw_k, \ba$ that minimizes the training error for the $m$ samples. Then, we simply output the hypothesis $h = \relu_{\bw_1, \dots, \bw_k, \ba}$.

Clearly, the algorithm is a proper learning algorithm. Furthermore, the running time is $2^{O(km)} poly(n, m) = 2^{O(k^5/\varepsilon^2)} poly(n, m, 1/\delta)$ as desired.


Thus, we are left to bound the error $\cL(h; \cD)$. To do this, first, observe that, from Theorem~\ref{thm:KST} and Theorem~\ref{thm:lips}, we have $\cR_m(1\text{-ReLU}(n)) \leq \frac{2}{\sqrt{m}}$. Hence, from Fact~\ref{fact:subadd}, we have $\cR_m(k\text{-ReLU}(n)) \leq \frac{2k}{\sqrt{m}}$. Now, observe that, for the region $[-k, k] \times [-k, k]$ the loss function $\ell(y', y) = (y' - y)^2$ is $(2k)$-Lipschitz in the first argument and $(4k^2)$-bounded. As a result, from Theorem~\ref{thm:gen-err}, the following holds for all $f \in k\text{-ReLU}(n)$ with probability at least $1 - \delta$:
\begin{align} \label{eq:gen-err-simple}
|\cL(f; \cD) - \cL(f; S)| \leq 4 \cdot (2k) \cdot \frac{2k}{\sqrt{m}} + 2 \cdot (4k^2) \cdot \sqrt{\frac{\log(1/\delta)}{m}} \leq \frac{\varepsilon}{2},
\end{align}
where the second inequality comes from our choice of $m$.

Let $f_{\opt} \in k\text{-ReLU}(n)$ be the minimizer of $\cL(f; \cD)$. Since $h$ minimizes the training error,
\begin{align} \label{eq:opt-simple}
\cL(h; S) \leq \cL(f_{\opt}; S).
\end{align}

As a result, we have
\begin{align*}
\cL(h; \cD) \stackrel{\eqref{eq:gen-err-simple}}{\leq} \cL(h; S) + \frac{\varepsilon}{2} \stackrel{\eqref{eq:opt-simple}}{\leq} \cL(f_{\opt}; S) + \frac{\varepsilon}{2} \stackrel{\eqref{eq:gen-err-simple}}{\leq} \cL(f_{\opt}; \cD) + \varepsilon = \left(\min_{f \in k\text{-ReLU}(n)} \cL(f; \cD)\right) + \varepsilon,
\end{align*}
which concludes the proof.
\end{proof}

The results above should be compared to those of ~\cite{goel2016reliably} who showed similar learnability results as above, except that their algorithm is \emph{improper}. That is, their algorithm would output a (modification of) low-degree polynomial, as opposed to sums of ReLUs (which our algorithm outputs).  We remark here that, while our algorithm is advantageous to their in this sense, their algorithm is faster extends to a larger class of networks.

\subsubsection{Properly Learning ReLUs: The Realizable Case}

\label{sec:learning-algo-pointwise}
Next we prove Theorem~\ref{thm:learning-realizable}. For the realizable setup we can get better guarantees by using the improved generalization bound from Theorem~\ref{thm:gen-err-realizable}.

\begin{proof}[Proof of Theorem~\ref{thm:learning-realizable}]
First, let us describe the algorithm. Given samples $S = \{(\bx_i, y_i)\}_{i \in [m]}$ such that
\begin{align*}
m = \left\lceil \frac{10^6 C \cdot k^2 \log^3(10C k/\eps)}{\eps} + \frac{8k^2\log(1/\delta)}{\eps} \right\rceil
\end{align*}
where $C$ is the constant from Theorem~\ref{thm:gen-err-realizable}.

We use the algorithm from Lemma~\ref{lem:exp-algo} to solve for $\bw_1, \dots, \bw_k, \ba$ that minimizes the training error for the $m$ samples. Then, we simply output the hypothesis $h = \relu_{\bw_1, \dots, \bw_k, \ba}$.

Clearly, the algorithm is a proper learning algorithm, and the running time is $2^{O(km)} poly(n, m) = 2^{O((k^3/\eps) \cdot \log^3(k/\eps))} poly(n, m, 1/\delta)$ as desired. Furthermore, since $S$ is realizable, we must have $\cL(h; S) = 0$.

Finally, we will apply the generalization bound from Theorem~\ref{thm:gen-err-realizable}. To do so, first recall from the proof of Theorem~\ref{thm:learning} that $\cR_m(k\text{-ReLU}(n)) \leq \frac{2k}{\sqrt{m}}$ and that, in the region $[-k, k] \times [-k, k]$ the squared loss function is $(4k^2)$-bounded. Furthermore, the squared loss is $2$-smooth. As a result, we may apply Theorem~\ref{thm:gen-err-realizable} which implies that, with probability $1 - \delta$, we have
\begin{align*}
\cL(h; \cD) \leq C \left(2 \log^3 m \cdot \frac{4k^2}{m} + \frac{4k^2 \cdot \log(1/\delta)}{m}\right)
\leq \eps,
\end{align*}
where the inequality follows from our choice of $m$.
\end{proof}

\end{document}